\documentclass{article}
\usepackage{graphicx} 
\usepackage{arxiv}

\usepackage[]{natbib}
\usepackage{appendix}
\usepackage{geometry}
\usepackage[dvipsnames]{xcolor}
\usepackage{placeins}
\usepackage{enumitem}

\usepackage{amsmath}
\usepackage{amssymb}
\usepackage{amsthm}
\usepackage{mathtools}
\usepackage{bbm} 

\usepackage{booktabs}
\usepackage{multirow}

\usepackage{hyperref}
\usepackage[capitalise]{cleveref}
\usepackage{subcaption}
\newcommand{\N}{\mathbb{N}}
\newcommand{\R}{\mathbb{R}}
\DeclareMathOperator{\E}{\mathbb{E}}
\DeclarePairedDelimiter{\abs}{\lvert}{\rvert}

\DeclareMathOperator*{\argmax}{\text{argmax}}

\newtheorem{definition}{Definition}
\newtheorem{theorem}{Theorem}
\newtheorem*{theorem*}{Theorem}
\newtheorem{corollary}{Corollary}
\newtheorem{lemma}{Lemma}
\newtheorem{proposition}{Proposition}

\hypersetup{
	colorlinks = true,
	urlcolor = {black},
	linkcolor = {niceBlue},
	citecolor = {niceRed} 
}

\definecolor{niceRed}{RGB}{190,38,38}
\definecolor{niceBlue}{HTML}{0466a7}

\renewcommand{\headeright}{}
\renewcommand{\undertitle}{}

\title{Mean-based algorithms: \\A lower bound and regret}
\renewcommand{\shorttitle}{Mean-based algorithms: A lower bound and regret}
\author{
    \textbf{Julius Durmann} \\
    Technical University of Munich \\
    \texttt{\href{mailto:julius.durmann@tum.de}{julius.durmann@tum.de}}
    \and 
    \textbf{Amelie Kleber} \\
    Technical University of Munich \\
    \texttt{\href{mailto:amelie.kleber@tum.de}{amelie.kleber@tum.de}}
}
\date{}

\begin{document}

\maketitle

\begin{abstract}
	Mean-based algorithms are online learning algorithms that assign low probability to actions with low average rewards. Recent research shows that they converge to serially undominated actions, which serve as approximations to Nash equilibria in economic games. However, empirical studies indicate that mean-based algorithms converge more slowly in bandit-feedback settings than established no-regret alternatives.

	This work investigates mean-based algorithms under unknown horizons and bandit feedback. In this setting, we provide the first lower bound on the algorithm-defining sequence $\gamma_t$, establishing a fundamental limit on the learning speed of such algorithms. In multi-armed bandit problems, this result constrains the rate at which any algorithm can reliably identify low-reward actions while acting according to this knowledge.

	We also propose two mean-based algorithms: one generalizes $\epsilon$-greedy, and the other extends mean-based Exp3 to unknown horizons.
	Our experiments show that mean-based algorithms, although slightly slower, can perform competitively with other bandit-feedback algorithms.

	We further study the relationship to regret. Depending on the choice of $\gamma_t$, the intersection with no-regret algorithms is non-trivial, and we show that some algorithms are both mean-based and no-regret.
\end{abstract}

\vfill
\pagebreak
\vfill

{
\pagenumbering{gobble}
\setlength{\parskip}{0.2\baselineskip}
\hypersetup{linkcolor=black}
\setcounter{tocdepth}{2}
\tableofcontents
}

\vfill
\pagebreak
\pagenumbering{arabic}
\setcounter{page}{1}

\section{Introduction}

In online learning, an agent interacts with an unknown environment by repeatedly selecting actions, observing the corresponding rewards, and potentially receiving additional feedback. The objective is to maximize the average reward over time.
Mean-based algorithms are developed for online learning problems and operate on the principle that actions with low historical payoffs are selected with low probability in subsequent rounds. 
This paper analyzes the algorithm's rate $\gamma_t$, which formalizes this principle by determining (1) a threshold for classifying actions as "low reward" and (2) an upper bound on the probability of selecting such actions. 
Interpreted differently, $\gamma_t$ characterizes the speed at which the algorithm can reliably identify and avoid low-reward actions, which can be understood as a core problem in online learning.

\subsection{Why should we study mean-based algorithms?}

We argue that there are good reasons to study online learning problems and mean-based algorithms. 

\begin{description}

    \item[Online learning.]
    The online learning problem is a simple model for sequential decisions under uncertainty. It has been applied with slight variations to different applications: from algorithmic pricing \citep{trovo_multi-armed_2015}, recommendation systems \citep{li_contextual-bandit_2010}, and decentralized networking \citep{appavoo_shrewd_2018}, to revenue maximization in auctions \citep{croissant_real-time_2020}, to name just a few.
    The concept of mean-based algorithms offers an intriguingly simple way to address online learning problems.

    \item[Established algorithms.] 
    \citet{braverman_selling_2018} showed that many established algorithms (such as Exp3, FTRL, and MWU) can be mean-based in the finite-horizon setting. These are common algorithms used to solve online learnig problems, and the mean-based approach helps us to understand them better.

    \item[A new perspective.] 
    In most cases, online learning algorithms are studied through the lens of regret, which measures the difference between the algorithm's strategy and a baseline strategy in hindsight. Being "no-regret" essentially provides a performance guarantee for the past behavior of the algorithm. 
    In contrast, mean-basedness captures what the algorithm will do in the next step, which is a statement about the immediate future. 
	The mean-based property thus offers us \emph{an alternative way to think about online learning algorithms} that adds more context to existing no-regret considerations.

    \item[Convergence.] 
    The new perspective allows for novel convergence results in games. The driving motivation behind this paper is the recent result by \citet{durmann_online_2026}, which identifies the outcome of repeated interaction among multiple mean-based algorithms with the set of correlated rationalizable actions \citep{brandenburger_rationalizability_1987} of the underlying game. 
    \begin{theorem*}[\citet{durmann_online_2026}, simplified]
        \label{thm:convergence-to-CR}
        If all players use mean-based algorithms in a finite normal-form game, their strategies converge almost surely to the set of correlated rationalizable actions in the limit.
    \end{theorem*}
    This statement yields meaningful convergence results for economic games that represent algorithmic pricing \citep{durmann_online_2026}. While the result does not provide an explicit convergence rate, its inductive proof suggests that the convergence rate is closely tied to the individual learners' rates $\gamma_t$.

\end{description}

In the context of online learning and algorithmic dynamics in games, this paper aims to better understand the limitations of mean-based algorithms. We address the following questions.

\begin{enumerate}
    \itshape
    \item What is the best asymptotic behavior (rate $\gamma_t$) that mean-based algorithms can achieve? 
	\item Which regret properties do mean-based algorithms have?
	\item What is their empirical performance?
\end{enumerate}

\subsection{Contributions of this paper}

We provide the first closer inspection of mean-based algorithms and their defining rate $\gamma_t$. We focus on the bandit-feedback unknown-horizon setting because it is broadly applicable to many scenarios, including multiplayer interactions with limited knowledge and no clear time horizon. Most notably, this includes practically relevant settings such as algorithmic pricing.
In summary, our contributions are the following.
\begin{enumerate}
	\item First and foremost, we derive a lower bound for the rate $\gamma_t$ of mean-based algorithms in the adversarial bandit-feedback setting: $\gamma_t \in \Omega(1/\sqrt{t})$. Our proof leverages the usual exploration-exploitation dilemma of MAB problems in this new context. To the best of our knowledge, this is the first lower bound for mean-based algorithms, providing insights into the performance levels they can achieve.
	\item We propose two mean-based algorithms that operate under bandit feedback with an unknown time horizon and analyze their mean-based rates $\gamma_t$.
	\item Through simulation, we analyze the performance of these algorithms in a stochastic multiarmed bandit environment and an economic normal-form game. We observe that, while slower than standard no-regret algorithms, mean-based algorithms can perform substantially better in terms of convergence speed in normal-form games than previously observed.
	\item We discuss the diverse relation between mean-based algorithms and regret, and show that there exist bandit-feedback algorithms that are both mean-based and no-external-regret.
	\item Lastly, we reconcile the two existing definitions of mean-based algorithms, which differ in how they treat the unknown rewards from the current time step. This observation contributes to a more coherent body of literature.
\end{enumerate}

\subsection{Related work}

Below, we provide additional contex for our paper.

\subsubsection{Online learning}

Mean-based algorithms operate in the context of online learning, where an algorithm chooses from a set of actions and receives a reward and feedback from the environment. Two major variants of online learning environments are the stochastic and the adversarial settings. In the former, rewards are drawn from a fixed but unknown distribution, whereas in the latter, an adversary selects the rewards. Stochastic environments have been modeled as early as the work by \citet{thompson_likelihood_1933}. The newer adversarial setting has been introduced much later by the works of \citet{freund_decision-theoretic_1997}, \citet{littlestone_weighted_1994}, and \citet{auer_gambling_1995}. 

One can typically distinguish algorithms based on their feedback oracle. While some algorithms require only the reward for their selected action (bandit feedback), others require more information, such as the gradient of their utility function (gradient feedback) or the counterfactual rewards for all actions (full feedback). 
Well-known bandit-feedback algorithms include the UCB algorithm \citep{auer_finite-time_2002} and the Exp3 algorithm \citep{auer_gambling_1995}. With full feedback, there are algorithms such as the Multiplicative Weights (MWU) or "Hedge" algorithm \citep{littlestone_weighted_1994, freund_decision-theoretic_1997}. For gradient feedback, Online Gradient Descent/Ascent, Mirror Descent/Ascent, and Follow-the-Regularized-Leader \citep{shalev-shwartz_online_2011} have been established. A great overview of many of these algorithms and the theory of bandit algorithms is provided in the books by \citet{lattimore_bandit_2020} and \citet{shalev-shwartz_online_2011}. 

With bandit feedback, online learning becomes the "multi-armed bandit" (MAB) problem. Inherent to this problem is the well-known "exploration-exploitation" dilemma, which captures the trade-off between taking known-to-be-good actions (exploitation) and sampling new or rarely chosen actions (exploration) to better understand the environment.

In both the stochastic and adversarial settings, much interest lies in the performance of the learning algorithm. This is usually measured in terms of (external) regret, where the algorithm's reward is compared with the best reward achievable by any fixed action in hindsight. Some algorithms, such as Exp3, have been shown to suffer sublinear regret over time \citep{auer_gambling_1995, audibert_minimax_2009}. A stronger baseline of performance is given by internal and swap regret, and there exist schemes that lift no-external regret algorithms to no-swap regret algorithms \citep{blum_external_2007,ito_tight_2020}.

In this paper, we will focus on bandit-feedback mean-based algorithms in stochastic and adversarial environments. Bandit feedback is particularly relevant to most real-world applications, such as algorithmic pricing, where counterfactual reward information is often unavailable.  We investigate the algorithms' defining rate $\gamma_t$ in adversarial environments, which affects their flexibility and convergence, and analyze their regret properties.

\subsubsection{Mean-based algorithms}

Mean-based algorithms have been introduced by \citet{braverman_selling_2018} as a class of online learning algorithms that follow the simple idea of playing actions with low past rewards with low probability. This characterization of the algorithms is agnostic to the setting (stochastic/adversarial) and to the type of feedback oracle, making the concept widely applicable. \citep{braverman_selling_2018} is also a first source for mean-based versions of many online learning algorithms in the finite-horizon case.

Many papers analyze the behavior of mean-based algorithms in repeated principal-agent games, in which a learning agent plays against an optimizing principal agent \citep{braverman_selling_2018, deng_strategizing_2019, kolumbus_auctions_2022, hartline_regulation_2025, lin_generalized_2025}.
In the context of repeated auctions, \citet{braverman_selling_2018} use the definition of the mean-based property to show how the auctioneer can exploit a bidder employing such learning algorithms. This strategy involves enticing the mean-based bidder with low initial prices, offering high rewards to induce aggressive bidding, before subsequently raising prices. Due to the accumulation of past high rewards, the algorithm continues to select the historically rewarding action despite the shift in payoff structure. Comparable dynamics appear in repeated contractual relations \citep{kolumbus_contracting_2024}, Bertrand competitions \citep{hartline_regulation_2025, arunachaleswaran_algorithmic_2024}, and more general repeated games \citep{deng_strategizing_2019, lin_generalized_2025}, but, different from auctions, both players may actually benefit in some instances. In the algorithmic space, we can view these games as Stackelberg games in which one agent chooses a (mean-based) algorithm, and the principal agent then decides which strategy to implement.

\subsubsection{Convergence of mean-based algorithms}

An orthogonal branch of the literature has focused on standard normal-form games in which all players commit to mean-based algorithms.

\citet{deng_nash_2022} demonstrate that in repeated first-price auctions, when at least two bidders employ mean-based learning strategies, play converges to the Nash equilibrium in the time-average, i.e., the fraction of rounds where the NE is played approaches one in the limit. Nevertheless, the research findings indicate that convergence is not guaranteed if only a single player adopts this learning rule. In Bertrand competition, the mutual implementation of mean-based algorithms yields a competitive Nash equilibrium \citep{arunachaleswaran_algorithmic_2024}.

Further results establish broader equilibrium properties of mean-based algorithms. \citet{kolumbus_auctions_2022} show that when such algorithms converge to a coarse correlated equilibrium, the resulting equilibrium is co-undominated, thereby avoiding outcomes involving dominated strategies. This property is particularly relevant for no-regret algorithms, as they are known to converge to the coarse correlated equilibria (CCE) of a game \citep{hart_regret-based_2003}. Furthermore, the result by \citet{durmann_online_2026} (\cref{thm:convergence-to-CR}) extends these insights by proving that mean-based algorithms converge in \textit{last-iterate} to strategy profiles in the correlated rationalizable (CR) set, which is equivalent to the strictly serially undominated set in finite games. In economic games such as Bertrand competitions, this often implies convergence to or near the Nash equilibria, as properties such as potentialness or supermodularity are sufficient to make the CR set small. This makes these algorithms relevant for predicting the outcome of learning behavior in economic games. 

Last-iterate convergence means that the agent's actual action distribution converges to the CR set. This is stronger than the known time-average convergence results to CCE for no-regret algorithms in general games, which only imply that the empirical distribution converges. 
The convergence guarantees are usually formulated in the limit $t \to \infty$, and the theoretical speed of convergence depends on the rate $\gamma_t$: Small $\gamma_t$ implies faster convergence. In experiments, \citet{durmann_online_2026} observed that their mean-based implementation of Exp3 (a bandit-feedback algorithm) indeed converges to the Nash equilibrium, but at a slow asymptotic rate and typically after a large number of iterations. We replicate some of these experiments using our algorithms, observing more stable outcomes and faster convergence.

\section{Preliminaries}

In this section, we first describe the online learning setting and our notation. 
We then define the class of mean-based algorithms and briefly reason about its properties.

\subsection{Setting and notation}

Before we formally define the algorithm class, let us introduce some notation and concepts.

We consider the following finite online learning problem. Let $[k] := \{1, \dots, k\}$ be the set of natural numbers up to $k$, and $\Delta(k) \subseteq \R^k$ denote the $(k-1)$-dimensional probability simplex. 
An agent at time $t$ selects actions $a_t \in \mathcal{A} := [k]$ sampled from a distribution $p_t \in \Delta(k)$ and receives reward $x_t(a_t)$. 
In the \textit{stochastic} case, the rewards are sampled independently from their corresponding time-invariant distribution at each step. For simplicity, we will assume that these distributions have bounded support on $[0, 1]$, which ensures that average rewards are well-defined. In the \textit{adversarial} setting, an adversary selects a sequence (an "environment") $(x_t)_{t \in \N}$ of rewards $x_t \in [0, 1]^k$ for $k \geq 2$ discrete actions and $t \in \N$. 

With \textit{bandit feedback}, the agent only observes its reward $x_t(a_t)$ to update its internal beliefs about the environment. With \textit{full feedback}, it also observes the entire vector $x_t$ after submitting its action. These experiences form the observations $o_t \in O_t$ of an algorithm at time $t+1$: With bandit feedback, $o_t = (a_1, x_1(a_1), \dots, a_t, x_t(a_t))$, while with full feedback, $o_t = (a_1, x_1, \dots, a_t, x_t)$. 

An \textit{algorithm} $Alg$ defines a mapping $O_t \to \Delta(k)$ for each step. This mapping from observations to distributions is deterministic, but the actions $a_{t+1} \sim p_{t+1} = Alg(o_t)$ are accordingly sampled at random.\footnote{In the stochastic setting, rewards are also random.}
As a result, observations $o_t$ are random variables. The stochastic process defined by $o_t$ depends on the algorithm and the environment. It can be measured by its natural filtration.

The most common objective in online learning problems is to minimize (average) regret.

\begin{definition}[Regret and no-regret algorithm]
	For a sequence of (random) actions $a_t$ by an algorithm, the expected external regret with respect to action $a$ at time $T$ is
	\begin{equation*}
		R_T(a) = \mathbb{E} \left[ \sum_{t = 1}^T x_t(a) - x_t(a_t) \right].
	\end{equation*}
	The algorithm is said to be no-external-regret (no-regret) if its expected regret is sublinear for every action $a$, i.e. $R_T(a) \in o(T)$.
\end{definition}

Let $a \in [k]$ be an action. We define the sum of \textit{counterfactual} rewards $\sigma_t(a)$ and the average \textit{counterfactual} reward $\alpha_t(a)$ for $a$:
\begin{equation*}
	\sigma_t(a) = \sum_{s = 1}^{t} x_s(a) \qquad \text{and} \qquad \alpha_t(a) = \frac{1}{t} \sigma_t(a).
\end{equation*}

\subsection{Definition of mean-based algorithms}

Mean-based algorithms have been introduced by \citet{braverman_selling_2018} as algorithms that \textit{"will rarely pick an arm whose current mean is significantly worse than the current best mean"}. 
Building on \citeauthor{braverman_selling_2018}'s definition, \citet{deng_nash_2022}, and \cite{durmann_online_2026} define mean-based algorithms for an unknown time horizon. Here, we will consider the same notion of mean-based algorithms, and we state the following definition(s):
\begin{definition}[Mean-based algorithm / predictive mean-based algorithm]
	\label{def:mean-based-algorithms}
	An algorithm is called 
	\begin{enumerate}[label=\roman*)]
		\item \emph{$\gamma_t$-mean-based} if, whenever $\alpha_t(a') - \alpha_t(a) > \gamma_t$ for some actions $a, a' \in \mathcal{A}$,
		\item \emph{predictive $\gamma_t$-mean-based} if, whenever $\alpha_{\textcolor{orange}{t + 1}}(a') - \alpha_{\textcolor{orange}{t + 1}}(a) > \gamma_t$ for some actions $a, a' \in \mathcal{A}$,
	\end{enumerate}
	the algorithm picks action $a$ at time $t + 1$ with probability $\mathbb{P}_{t + 1}(a_{t+1} = a) \leq \gamma_t$. 
	An algorithm is (predictive) \emph{mean-based} if it is (predictive) $\gamma_t$-mean-based for some monotonically decreasing sequence $\gamma_t$ such that $\gamma_t \to 0$ as $t \to \infty$.
\end{definition}
In the remainder of this paper, we will often phrase the condition in the definition by means of $\sigma_t$ instead of $\alpha_t$: $\sigma_t(a') - \sigma_t(a) > \gamma_t \cdot t$. We will also call corresponding actions $a$ "dominated".

\subsubsection{Equivalence of definitions}

Note how the two definitions (mean-based vs. predictive mean-based) differ in the time frame of considered rewards. 
The predictive version, which is closer to the formulation by \citet{braverman_selling_2018} and \citet{kolumbus_contracting_2024}\footnote{Their definitions are also formulated for finite time horizons.}, assumes knowledge about the upcoming reward vector $x_{t + 1}$, while the non-predictive variant mirrors the definitions by \citet{durmann_online_2026, deng_nash_2022} and \citet{lin_generalized_2025}. Nonetheless, we can show that the two definitions can be mutually converted.

\begin{proposition}
	\label{proposition:equivalence-of-mean-based-definitions}
	\leavevmode
	\begin{enumerate}
		\item Assume an algorithm is predictive $\gamma_t$-mean-based.  
		Then the algorithm is also $\gamma_t'$-mean-based with a rate $\gamma_t' = \frac{\gamma_t + 1 / (t + 1)}{1 - 1 / (t + 1)}$. 
		\item Assume an algorithm is $\gamma_t$-mean-based. 
		Then the algorithm is also predictive $\gamma_t'$-mean-based with a rate $\gamma_t' = \frac{\gamma_t + 1/t}{1 + 1/t}$.
	\end{enumerate}
\end{proposition}

We provide the proof in \cref{sec:proof-predictive-equals-non-predictive-mean-based}. In the remainder of the paper, we will focus on the first definition: (non-predictive) $\gamma_t$-mean-based algorithms.

\subsubsection{Interpretation of the mean-based property}

Observe that $\gamma_t$ in \cref{def:mean-based-algorithms} has two different purposes: It acts as a threshold for deciding which actions should be avoided, and it determines how much probability can be put on these actions. 
The first part means that the algorithm is able to \emph{identify} low-reward actions up to a resolution of $\gamma_t$, while the latter part implies that the algorithm \emph{acts consistently} with these insights. Being mean-based can thus be understood as a "reliability guarantee" when it comes to the behavior of the algorithm, with $\gamma_t$ being the "reliability certificate".

We note that mean-basedness is trivial under full feedback: The learner has exact knowledge about past rewards and can act accordingly. Under bandit feedback, however, the problem becomes more challenging as identification and consistency can be in conflict. This is what we use to derive our lower bound in \cref{sec:lower-bound}.

By definition, mean-based algorithms identify and reject actions with low rewards. Equivalently, they identify and favor actions with high rewards, which we formalize by the following proposition. This suggests that mean-based algorithms encode another intuitive trait of algorithms - reliably selecting good (past) actions.

\begin{proposition}
	At time $t$, determine action $a_t^* \in \argmax_{a \in \mathcal{A}} \alpha_t(a)$ and define 
	\begin{equation*}
		\mathcal{A}_{t,high} = \{ a \in \mathcal{A} \mid \alpha_t(a_t^*) - \alpha_t(a) \leq \gamma_t \}. 
	\end{equation*}
	If the algorithm is $\gamma_t$-mean-based, then $\mathbb{P}_{t+1}(a_{t+1} \in \mathcal{A}_{t,high}) \geq 1 - k \gamma_t$. \\
	Similarly, if $\mathbb{P}_{t+1}(a_{t+1} \in \mathcal{A}_{t,high}) \geq 1 - \gamma_t$ for all $t \geq 1$, the algorithm is $\gamma_t$-mean-based.
\end{proposition}

\begin{proof}
	Assume that the algorithms is $\gamma_t$-mean-based. Then we can partition $\mathcal{A}$ into $\mathcal{A}_{t, high}$ and $\mathcal{A}_{t, low} = \mathcal{A} \setminus \mathcal{A}_{t, high}$ (which may be empty). All actions in $\mathcal{A}_{t, low}$ are dominated in the mean-based sense, so $\mathbb{P}_{t+1}(\mathcal{A}_{t,low}) \leq k \gamma_t$ by definition and application of the union bound. The result follows from $\mathbb{P}_{t+1}(\mathcal{A}_{t,low}) + \mathbb{P}_{t+1}(\mathcal{A}_{t,high}) = 1$. 

	Assume $\mathbb{P}_{t+1}(a_{t+1} \in \mathcal{A}_{t,high}) \geq 1 - \gamma_t$. Then, $\mathbb{P}_{t+1}(\mathcal{A}_{t,low}) \leq \gamma_t$, so each dominated action has probability of at most $\gamma_t$.
\end{proof}

We note that the present definition explicitly couples the identification threshold with the behavioral limitation by using the same rate $\gamma_t$. In principle, this is not strictly necessary: For future work, one could consider $(\eta_t, \epsilon_t)$-mean-based algorithms with a reward threshold of $\eta_t$ and a probability of at most $\epsilon_t$ of playing dominated actions. 
In this work, we stick with the notion of $\gamma_t$-mean-based algorithms because we are usually interested in both properties simultaneously and because their explicit coupling yields clear insights under bandit feedback.

\section{Mean-based algorithms with bandit feedback}
\label{sec:mean-based-algorithms-with-bandit-feedback}

The definition of mean-based algorithms suggests that full feedback is required. This is not the case, however, as \citet{braverman_selling_2018} have shown in the finite-horizon setting with a variant of Exp3. We present two mean-based algorithms that use bandit feedback for the unknown-horizon setting. 


When talking about the $\gamma_t$ rate of an algorithm, we will focus on the long-term behavior. For a finite number of initial steps, we can set $\gamma_t = 1$ since $p_{t+1}(a) \leq 1$ holds trivially, so we may ignore any initial fluctuations and non-monotone behavior. To not clutter notation, we will not make this explicit in the following sections.

\subsection{Generalization of the $\epsilon$-greedy algorithm}

We propose the following algorithm scheme that is directly inspired by the definition of mean-based algorithms: 
Let $\eta_t \geq 0$ and $\epsilon_t \geq 0$ be monotonically decreasing sequences with $\epsilon_t \leq 1$, and let $\hat{\sigma}_t(a) = \sum_{\tau = 1}^{t} \hat{x}_\tau(a)$ be the estimated sum of rewards. 
For the bandit-feedback setting, we also require $\epsilon_t > 0$.
With full feedback, we set $\hat{x}_t(a) = x_t(a)$, while we use $\hat{x}_t(a) = \frac{x_t(a)}{p_t(a)}\cdot \mathbbm{1}[a_t = a]$ if only bandit feedback is available. We note that this common estimator of rewards is conditionally unbiased given the observations, $\mathbb{E}[\hat{x}_t \mid o_{t-1}] = x_t$.

The algorithm picks action $a$ at time $t + 1$ with probability
\begin{equation*}
	p_{t+1}(a) = \begin{cases}
	\frac{\epsilon_t}{k}, &  \text{if} ~ \exists a': \hat{\sigma}_t(a') - \hat{\sigma}_t(a) > \eta_t \cdot t \\
	\frac{1}{k - n_{dom}} \cdot (1 - \frac{n_{dom}}{k} \cdot \epsilon_t), &  \text{else},
\end{cases}
\end{equation*}
where $n_{dom} = \vert \mathcal{A}_{dom} \vert$ is the number of dominated actions in the set
\begin{equation*}
	\mathcal{A}_{dom} = \{ a \in \mathcal{A} \mid \exists a' \in \mathcal{A}: \hat{\sigma}_t(a') - \hat{\sigma}_t(a) > \eta_t \cdot t \}.
\end{equation*}
Observe how this algorithm becomes an $\epsilon_t$-greedy bandit when we set $\eta_t = 0$.
$\epsilon_t \leq 1$ ensures that an undominated action will always be played with at least the same probability as any dominated action. 

Related to \cref{def:mean-based-algorithms}, this algorithm decides which actions to avoid based on a threshold sequence $\eta_t$, and puts only a small probability $\epsilon_t$ on them.
It is easy to see that, with full feedback, we can choose $\eta_t = \epsilon_t = \gamma_t$ for any mean-based sequence $\gamma_t$ to find a $\gamma_t$-mean-based algorithm. In the bandit setting, we can show that the algorithm is mean-based under suitable choices of $\eta_t$ and $\epsilon_t$. 
\begin{proposition}
	\label{prop:mean-based-bandit-algorithm}
	Let $\eta_t$, $\epsilon_t$ be decreasing sequences with $\epsilon_t \in \omega(\sqrt{\log(t)/t})$ and let $p > 0$. 
	The algorithm is $\gamma_t$-mean-based with rate
	\begin{equation*}
		\max\{\eta_t, \epsilon_t\} + \frac{k \cdot \sqrt{\log(t)}}{\epsilon_t \cdot \sqrt{t}} \cdot \sqrt{2 p} + \frac{1}{t^p}.
	\end{equation*}
	For $\eta_t \leq \epsilon_t = \min\{1, \sqrt[4]{\log(t) / t}\}$ and $t$ such that $\epsilon_t < 1$, this implies
	\begin{equation*}
		\gamma_t = (1 + \sqrt{2p} \cdot k) \cdot \sqrt[4]{\frac{{\log(t)}}{t}} + \frac{1}{t^p},
	\end{equation*}
	which decreases monotonically for $t \geq 2$ and satisfies $\gamma_t \to 0$ for $t \to \infty$.
\end{proposition}

For the proof, we borrow ideas from \citep{braverman_selling_2018}\footnote{See the online appendix of their paper.}. The idea is that, for sufficiently large exploration $\epsilon_t$, it is unlikely that our algorithm's estimates of the rewards $\hat{x}$ are too bad. If $\epsilon_t$ still decays over time, we find that the algorithm plays dominated actions with decreasing probability. We provide the formal proof in \cref{sec:proofs-mean-based-rates}.

\subsection{Mean-based Exp3 with unknown time horizon}
\label{sec:mean-based-Exp3}

We also propose a variant of the well-known Multiplicative Weights (MWU) / Hedge algorithm and the Exp3 algorithm \citep{auer_gambling_1995}. Let $\eta_t$, $\epsilon_t$ be two decreasing sequences with $\epsilon_t \leq 1$. The algorithm plays action $a \in [k]$ with probability
\begin{equation*}
	p_{t + 1}(a) = (1 - \epsilon_t) \cdot \frac{w_t(a)}{\sum_{a'} w_t(a')} + \frac{\epsilon_t}{k},
	\qquad \text{where} \qquad
	w_t(a) = \exp \left( \eta_t \cdot \hat{\sigma}_t(a) \right).
\end{equation*}
As before, the definition allows for both a full-feedback algorithm and a bandit-feedback algorithm. The availability of information affects the estimates $\hat{x}_t$, and in turn the estimated sum of rewards $\hat{\sigma}_t$.

Again, we show that the algorithms are mean-based.
We defer the proof, which builds on the proof for the finite-horizon case by \citet{braverman_selling_2018}, to \cref{sec:proofs-mean-based-rates}.
\begin{proposition}
	\label{prop:mean-based-exp3}
	\leavevmode
	\begin{enumerate}
		\item 
		The full-feedback algorithm (MWU/Hedge) with $\eta_t = \epsilon_t = 1 / \sqrt{t}$ is $\gamma_t$-mean-based with rate 
		\begin{equation*}
			\gamma_t = \frac{\log(\sqrt{t} + 5)}{\sqrt{t}},
		\end{equation*}
		which is monotonically decreasing for $t > 0$ and converges to zero as $t \to \infty$.
		
		\item 
		Let $\eta_t$, $\epsilon_t$ be two decreasing sequences such that $\eta_t, \epsilon_t \in \omega(\sqrt{\log(t) / t})$, $\epsilon_t \leq 1$, and $\epsilon_t \to 0$ for $t \to \infty$.
		Then the bandit-feedback algorithm (Exp3) is $\gamma_t$-mean-based with rate
		\begin{equation*}
			\gamma_t = 
			\frac{\sqrt{8 \log (t)} \cdot k}{ \epsilon_t \cdot \sqrt{t} } + \frac{\sqrt{\log(t)}}{\eta_t \cdot \sqrt{t}} + \frac{\epsilon_t}{k} + \mathcal{O}(1/t)
		\end{equation*}
		
		In particular, with $\eta_t = \epsilon_t = \min\{1, \sqrt[4]{\log(t) / t} \}$, we get
		\begin{equation*}
			\gamma_t = (2 \sqrt{2} \cdot k + 2) \cdot \sqrt[4]{\frac{\log(t)}{t}} + \mathcal{O}(1/t)
		\end{equation*}
		for $t$ such that $\epsilon_t < 1$, which means that $\gamma_t \in \mathcal{O}(\sqrt[4]{\log(t) / t})$. For $t \geq 2$, this decreases monotonically and converges to zero.
	\end{enumerate}
\end{proposition}

\section{Lower bound}
\label{sec:lower-bound}

We observe that the derived bounds on $\gamma_t$ in both algorithms scale with $t^{-1/4}$. We are able to show the following lower bound on $\gamma_t$, which scales with $t^{-1/2}$. To the best of our knowledge, this is the first result on the theoretical limit for mean-based algorithms.

\begin{theorem}
	\label{thm:mean-based-lower-bound}
	Let $\gamma_t' = \gamma_0 / (t + 2)^\alpha$. If $\gamma_0 \leq \frac{1}{2}$ and $\alpha \geq \frac{1}{2}$ and at least one of these inequalities is strict, then $\gamma_t'$ is \textbf{not} a valid mean-based rate in the adversarial bandit-feedback setting. Moreover, any rate $\gamma_t \in o(1/\sqrt{t})$ is not a valid mean-based rate.
\end{theorem}

Note that the rate $\gamma_t^* = 1/(2 \sqrt{t + 2})$ itself is not ruled out by this theorem. However, any sequence that is "smaller" or decays faster is stated to be impossible.

Intuitively, the exploration-exploitation dilemma seems to be the root cause for restrictions on $\gamma_t$: For a small $\gamma_t$, we need to reduce exploration quickly. However, this makes our estimation prone to error and increases the probability of unknowingly playing dominated actions, thus potentially increasing $\gamma_t$ again. 
Here, we will make this tradeoff explicit by providing two adversarial environments that restrict the rate $\gamma_t$ of any mean-based algorithm with bandit feedback. We will then use them to derive an inequality that any $\gamma_t$ needs to satisfy, and finally show that only rates $\gamma_t \in \Omega(\gamma_t^*) = \Omega(1/\sqrt{t})$ can do so.

We structure our proof into five steps:
\begin{description}
	\item[\cref{sec:lower-bound-environment}:] We define two adversarial environments that capture the exploration-exploitation dilemma that mean-based algorithms face with bandit feedback.
	\item [\cref{sec:lower-bound-inequality}:] Based on these environments, we derive the inequality that any mean-based $\gamma_t$ rate needs to satisfy. In particular, any mean-based rate must be sufficiently large according to this inequality, which allows us to reject $\gamma_t$ sequences that are too small.
	\item[\cref{sec:lower-bound-approximation}:] We approximate one side of the inequality by an upper bound. While this is less restrictive, we can still reject many $\gamma_t$ rates.
	\item[\cref{sec:lower-bound-evaluation}:] By inserting the sequence $\gamma_t^*$, we observe that it satisfies the inequality ever so slightly.
	\item[\cref{sec:lower-bound-variation}:] Based on this observation, we confirm that small variations of the rate will break the inequality. This establishes our result.
\end{description}

\subsection{Adversarial selection of environments}
\label{sec:lower-bound-environment}

We start by defining two environments with $k \geq 2$ actions.
Consider the two reward sequences, expressed as vectors $(x_{t, 1}, \dots, x_{t, k})$,

\begin{equation*}
\begin{aligned}
	&x_t'(a) = \begin{cases}
		\begin{aligned}
			&(& 0   &, & 0, && 0, && \dots, && 0), &&& t \leq t_0 \\
			&(& 1/2 &, & \textcolor{orange}{0}, && 0, && \dots, && 0), &&& t > t_0
		\end{aligned}
	\end{cases}
	\\[1em] &\qquad \text{and} \\[1em]
	&x_t''(a) = \begin{cases}
		\begin{aligned}
			&(&0    &,  & 0, &&& 0, && \dots, && 0), &&& t \leq t_0 \\
			&(&1/2  &,  & \textcolor{orange}{0}, &&& 0, && \dots, && 0), &&& t_0 < t \leq t_1 \\
			&(&1/2  &,  & \textcolor{orange}{1}, &&& 0, && \dots, && 0), &&& t > t_1
		\end{aligned}
	\end{cases}
\end{aligned}
\end{equation*}

We will sometimes use $x_t'$ and $x_t''$ as synonyms for the respective environments. The idea behind these sequences is that, at some time $t_1 > t_0$, the first action dominates all other actions in the mean-based sense. From this moment on, a mean-based algorithm must play $a = 1$ in the first case ($x_t'$), while, in the second case ($x_t''$), $a = 2$ becomes the better action. However, when always choosing $a \neq 2$ after $t_1$, which will happen with high probability, it cannot distinguish between $x_t'$ and $x_t''$, so it will need to continue playing $a = 1$ with high probability (formal argument follows below). With $x_t''$, it will eventually (at time $t_c + 1 > t_1$) run into a situation in which it should play $a = 1$ with a low probability (less than $\gamma_{t_c}$), which can lead to a contradiction if the sequence $\gamma_t$ is too small.

Essential to this argument are two components of our setting: bandit feedback and an unknown time horizon. Bandit feedback is what makes exploration necessary. With full feedback, we already know that any rate $\gamma_t$ can be realized, for instance by means of the $\epsilon$-greedy algorithm. The unknown time horizon ensures that the times $t_1$ and $t_c$ exist (more on that below).

The differences between $a = 1$ and $a = 2$ (including counterfactual rewards) are as follows:
\begin{equation*}
	\begin{aligned}
		\Delta \sigma'(t) &:=& \sigma_t'(2) - \sigma_t'(1) &= -\frac{1}{2} (t - t_0) \\
		\Delta \sigma''(t) &:=& \sigma_t''(2) - \sigma_t''(1) &= -\frac{1}{2} (t - t_0) + \max \{ 0, t - t_1 \}.
	\end{aligned}
\end{equation*}
\cref{fig:mean-based-barrier} visualizes these utility differences and the relevant mean-based boundaries.

The times $t_1$ for which $\Delta \sigma'(t_1) < - \gamma_{t_1} \cdot t_1$ and $t_c$ for which $\Delta \sigma''(t_c) > \gamma_{t_c} \cdot t_c$ are guaranteed to exist, as $\gamma_t \to 0$ ensures that $\abs{\gamma_t \cdot t} < c \cdot t$ for any $c > 0$ and sufficiently large $t$. In particular, we can choose any $t_1$ for which $\gamma_{t_1} < \frac{1}{2}$ by defining $t_0$ such that $- \gamma_{t_1} \cdot t_1 - 1/2 \leq \Delta \sigma'(t_1) < - \gamma_{t_1} \cdot t_1$.

We then pick time $t_c$ to be the smallest integer number such that
\begin{equation}
	\label{eq:t_c}
	t_c - t_1 > 2 ( \gamma_{t_1} \cdot t_1 + \gamma_{t_c} \cdot t_c ) + 1.
\end{equation}
There is exactly one such number as the right side grows sublinearly ($\gamma_t$ is decreasing).
This choice of $t_c$ ensures that $\Delta \sigma''(t_c) > \gamma_{t_c} \cdot t_c$, so the algorithm will need to select $a = 1$ with probability of at most $\gamma_{t_c}$ in $x_t''$:

\begin{equation*}
		\Delta \sigma''(t_c) = \frac{1}{2} (t_c - t_1) + \Delta \sigma''(t_1) = \frac{1}{2} (t_c - t_1) + \Delta \sigma'(t_1) 
		\geq \frac{1}{2}(t_c - t_1) - \gamma_{t_1} \cdot t_1 - \frac{1}{2} 
		> \gamma_{t_c} \cdot t_c.
\end{equation*}

\begin{figure}[h]
	\centering
	\includegraphics{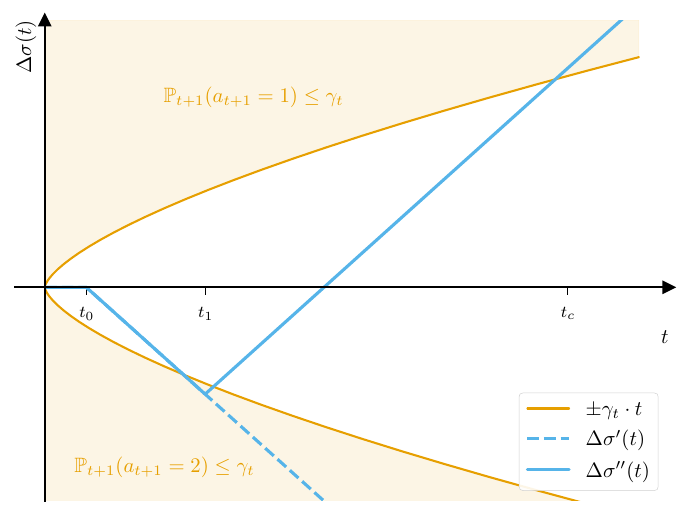}
	\caption{The utility differences $\sigma'(t)$ and $\sigma''(t)$ (blue) and the bound $\pm \gamma_t \cdot t$ (orange). If the difference crosses either $- \gamma_t \cdot t$ (at $t_1$) or $\gamma_t \cdot t$ (at $t_c$), a $\gamma_t$-mean-based algorithm must set the probability of the corresponding dominated action to at most $\gamma_t$.}
	\label{fig:mean-based-barrier}
\end{figure}

\subsection{Inequality for mean-based algorithms}
\label{sec:lower-bound-inequality}

With the definitions of the two environments, we can derive an inequality that specifies a condition on the mean-based rate in the bandit setting.
Let $\mathcal{X}_t = (x_1, \dots, x_t)$ denote the sequence of reward \textit{vectors} up to time $t$. This sequence includes all the information required to judge what a mean-based algorithm should do, but it also contains counterfactual information that the bandit-feedback algorithm has no access to.
We derive a bound for the probability that the algorithm plays $a_{t_c + 1} = 1$ in environment $x_t''$. Since the reward sequences are fixed, the only randomness that we need to reason about comes from sampling the actions.

\begin{equation*}
	\begin{aligned}
		\mathbb{P}(a_{t_c + 1} = 1 \mid \mathcal{X}''_{t_c})
		&\geq \mathbb{P}(a_{t_1 + 1} \neq 2, \dots, a_{t_c} \neq 2, a_{t_c + 1} = 1 \mid \mathcal{X}''_{t_c}) \\
		&= \mathbb{P}(a_{t_1 + 1} \neq 2, \dots, a_{t_c} \neq 2, a_{t_c + 1} = 1 \mid \mathcal{X}'_{t_c}) \\
		&= 1 - \mathbb{P}(a_{t_1 + 1} = 2 ~ \vee ~ \dots ~ \vee ~ a_{t_c} = 2 ~ \vee ~ a_{t_c + 1} \neq 1 \mid \mathcal{X}_{t_c}') \\
		&\geq 1 - \left( \sum_{t = t_1}^{t_c - 1} \mathbb{P}(a_{t + 1} = 2 \mid \mathcal{X}_{t_c}') + \mathbb{P}(a_{t_c + 1} \neq 1 \mid \mathcal{X}_{t_c}') \right) \\
		&\geq 1 - \left( \sum_{t = t_1}^{t_c - 1} \gamma_t + (k - 1) \gamma_{t_c} \right)
	\end{aligned}
\end{equation*}

In the second step, we replaced $\mathcal{X}''_{t_c}$ with $\mathcal{X}'_{t_c}$, which we are allowed to do because, under bandit feedback, \textit{the algorithm cannot distinguish between the two sequences if $a_{t_1 + 1}, \dots, a_{t_c} \neq 2$} and thus needs to put the same probability on these actions. More formally speaking, we replace the stochastic process defined by the algorithm and environment $x_t'$ by the stochastic process defined by the algorithm and environment $x_t''$. The equivalence holds because the specific bandit-feedback observations $o_t$, $t \leq t_{c}$, that we consider here have the same probability under both processes: Probabilities over actions are formed by the algorithm, which maps equivalent inputs, determined by $a_t, t \leq t_c$, to equivalent probabilities, and the rewards are equivalent for the actions considered in the derivation. Since $a_{t_c+1}$ is independent of $x_{t_c+1}$, the conditional marginals are equivalent, too.

We then used the complementary event of $(a_{t_1 + 1} \neq 2,  \dots a_{t_c} \neq 2, a_{t_c + 1} = 1)$ and the union bound to arrive at step four. The last step uses the definition of $x_t'$ and the assumption that the algorithm is $\gamma_t$-mean-based. In particular, since all actions are dominated by $a = 1$ in $x_t'$, their probability is less than $\gamma_t$. We also applied the union bound once again.

With $x_t''$, we need $\mathbb{P}(a_{t_c + 1} = 1 \mid \mathcal{X}''_{t_c}) \leq \gamma_{t_c}$ due to the mean-based property. In summary, we get the following requirement on the sequence $\gamma_t$. For all $t_1$ such that $\gamma_{t_1} < 1/2$:
\begin{equation}
	\label{eq:mean-based-inequality}
	\boxed{
		\gamma_{t_c} \geq \mathbb{P}(a_{t_c + 1} = 1 \mid \mathcal{X}''_{t_c}) \geq 1 - \sum_{t = 1}^{t_c - 1} \gamma_t - (k - 1) \gamma_{t_c} \iff \sum_{t = t_1}^{t_c - 1} \gamma_t + k \cdot \gamma_{t_c} \geq 1
	}
\end{equation}

This inequality highlights three linked effects that limit the sequence $\gamma_t$.

\begin{enumerate}
	\item If $\gamma_t$ becomes small, the elements of the sum won't sum up to one. This is a result of the insufficient exploration of the algorithm.
	\item As $\gamma_t$ becomes small, the distance between $t_1$ and $t_c$ shrinks, also lowering the sum. This is a direct result of the low $\gamma_t$ rate, which enforces strong decisions.
	\item As $\gamma_{t_c}$ becomes small, later values of $\gamma_t$ for $t \geq t_c$ are bounded by $\gamma_{t_c}$. This, in turn, affects the inequality for $t_1' > t_1$.
\end{enumerate}

Since the inequality needs to hold for all sufficiently large $t_1$, we find that the tail of the sequence $\gamma_t$ must always be sufficiently large. This implies that $\sum_{t = 1}^{\infty} \gamma_t = \infty$, which already rules out sequences such as $\gamma_t = 1/t^2$.
In the following, we use \cref{eq:mean-based-inequality} to derive stricter restrictions on $\gamma_t$.

We note that, for a given $\gamma_t$ sequence, the left-hand side of \cref{eq:mean-based-inequality}, $\sum_{t = t_1}^{t_c} \gamma_t + \gamma_{t_c}$, essentially only depends on $t_1$ as $t_c$ is a direct result of its value and $\gamma_t$ (see \cref{eq:t_c}). For convenience, we will thus write the left-hand side as a function
\begin{equation}
\label{eq:LHS-function}
	f_{\gamma_t}(t_1) := \sum_{t_1}^{t_c(t_1) - 1} \gamma_{t} + k \gamma_{t_c(t_1)}.
\end{equation}

If we can show that $f_{\gamma_t}(t_1) < 1$ for some valid choice of $t_1$, we know by our inequality that $\gamma_t$ is not a possible rate for the bandit setting. We formalize this finding in the following theorem.

\begin{theorem}
	\label{thm:functional-mean-based-condition}
	Let $(\gamma_t)_{t \in \N}$ be a monotonically decreasing sequence of values $\gamma_t \geq 0$ such that $\lim_{t \to \infty} \gamma_t = 0$. If for a fixed $\gamma_t$, there exists a $t_1 > 0$ such that
	\begin{equation*}
		\gamma_{t_1} < \frac{1}{2} \qquad \text{and} \qquad f_{\gamma_t}(t_1) < 1,
	\end{equation*}
	then $\gamma_t$ is not a valid mean-based rate in the bandit setting.
\end{theorem}

\begin{proof}
	The result follows from the definition of $f_{\gamma_t}$: If its value is strictly smaller than 1 for some $t_1$, we have found a sequence of rewards $x'$ and $x''$ (implied by $t_1$) such that the algorithm will choose a dominated action with probability strictly greater than $\gamma_t$ in one of these environments.
\end{proof}

Observe that the function $f_{\gamma_t}$ is monotone in $\gamma_t$ in the sense that $f_{\gamma_t'}(t_1) \leq f_{\gamma_t}(t_1)$ if $\gamma_t' \leq \gamma_t$ for all $t \geq t_1$. This follows from the definitions of $f_{\gamma_t}$ and $t_c$: For the same $t_1$, $t_c$ can be at most as high with $\gamma_t'$ as it is with $\gamma_t$ according to \cref{eq:t_c}, and all terms of $f_{\gamma_t'}$ (\cref{eq:LHS-function}) can be at most as high as their counterparts in $f_{\gamma_t}$, too.
In particular, we can analyze the limit $\lim_{t_1 \to \infty} f_{\gamma_t}(t_1)$ and check if its value is above or below 1. 
If the limit is indeed smaller than 1, we can conclude that any sequence $\gamma_t' < \gamma_t$, for sufficiently large $t$, will also have $\lim_{t_1 \to \infty} f_{\gamma_t'}(t_1) < 1$. Again, we summarize this finding in the following corollary.
\begin{corollary}
	\label{cor:mean-based-sequence-relations}
	Assume that $\lim_{t_1 \to \infty} f_{\gamma_t}(t_1) < 1$.
	Then any sequence $\gamma'_t$ for which there exists a $t'$ such that $\gamma'_t \leq \gamma_t$ for all $t \geq t'$ also cannot be a valid mean-based rate.
\end{corollary}
This finding will prove useful for stating our main result, as it tells us that we only need to find one invalid $\gamma_t$ sequence to rule out many other $\gamma_t'$ sequences.

In summary, we now have a small set of tools that we can use to reject entire families of $\gamma_t$ rates. To get a good bound, we now need to find a rate $\gamma_t$ for which $f_{\gamma_t}$ is as close to 1 as possible (in the limit) while being as large as possible in general. We will see that $\gamma_t^* = 1 / (2 \sqrt{t + 2})$ is exactly that: we can show that $f_{\gamma_t^*}$ is upper-bounded by a function that converges to 1, making it an ideal edge case of \cref{thm:functional-mean-based-condition}. 

\subsection{Approximation of the bound}
\label{sec:lower-bound-approximation}

In the following, we derive this upper bound on $f_{\gamma_t}$ in two steps. We will first find an approximation of $t_c$. Then, we will bound $f_{\gamma_t}$ from above by a closed-form function in $t_1$ based on an integral.

\subsubsection{Bound on $t_c$}

Recall that we chose $t_c$ to be the smallest integer number such that $t_c - t_1 > 2 ( \gamma_{t_1} \cdot t_1 + \gamma_{t_c} \cdot t_c ) + 1$, which is an implicit equation. Here we derive two conservative upper bounds for $t_c$, which we call $t_{c, max}$, in explicit form. Consider the candidate $\hat{t}_c > t_1$ for which

\begin{equation*}
	\hat{t}_c = \frac{1 + 2 \gamma_{t_1}}{1 - 2 \gamma_{t_1}} \cdot (t_1 + 1).
\end{equation*}

This is well-defined since $\gamma_{t_1} < \frac{1}{2}$ by construction.
Because $\gamma_{t_1} \geq \gamma_{\hat{t}_c}$ and $1 + 2 \gamma_{t_1} > 1$, this implies the following:
\begin{equation*}
	\hat{t}_c \geq \frac{1 + 2 \gamma_{t_1}}{1 - 2 \gamma_{\hat{t}_c}} \cdot (t_1 + 1)
	\implies  \hat{t}_c (1 - 2 \gamma_{\hat{t}_c}) > t_1 (1 + 2 \gamma_{t_1}) + 1
	\iff  \hat{t}_c - t_1 > 2 ( \gamma_{t_1} \cdot t_1 + \gamma_{\hat{t}_c} \cdot \hat{t}_c ) + 1
\end{equation*}

$\hat{t}_c$ hence satisfies our condition on $t_c$, so we know $t_c \leq \lceil \hat{t}_c \rceil$. We make this explicit with the following conservative upper bounds:
\begin{equation*}
	t_{c, max} = \left\lceil \frac{1 + 2 \gamma_{t_1}}{1 - 2 \gamma_{t_1}} \cdot (t_1 + 1) \right\rceil \in \N \qquad \text{or} \qquad t_{c, max} = \frac{1 + 2 \gamma_{t_1}}{1 - 2 \gamma_{t_1}} \cdot (t_1 + 1) + 1 \in \R^+
\end{equation*}
The first version is useful for numerical evaluation of $f_{\gamma_t}$ while the second version is better for analysis. Note that the ceil operator yields a lower or equal $t_{c, max}$ compared to adding $+1$.
Even if the ratio $\frac{1 + 2 \gamma_{t_1}}{1 - 2 \gamma_{t_1}}$ converges to 1 for large $t_1, t_c$, this does not mean that the distance between $t_c$ and $t_1$ becomes smaller. In particular, $t_1$ can grow faster than the decay of the ratio.

\subsubsection{Bound on $f_{\gamma_t}$ by an integral}

If we are given $\gamma_t$ in functional form, say $\gamma: \R \to \R^+$, we can use the fact that $\gamma$ is monotonically decreasing to derive an upper and lower bound of its value:
\begin{equation*}
	\label{eq:integral-upper-bound}
	\int_{t_1}^{t_c} \gamma(t) ~ dt + k \gamma(t_c) \leq \sum_{t = t_1}^{t_c - 1} \gamma(t) + k \gamma(t_c) \leq \int_{t_1}^{t_c} \gamma(t-1) ~ dt + k \gamma(t_c) \leq \int_{t_1}^{t_{c, max}} \gamma(t-1) ~ dt + k \gamma(t_c).
\end{equation*}
The upper bound is more important for our analysis, as it helps us show that $f_{\gamma_t}(t_1) < 1$. Of course, this is only useful if a closed-form integral exists. Otherwise, we could just stick with the sum from the start.


\subsection{Evaluation of the bound for $\gamma_t^*$}
\label{sec:lower-bound-evaluation}

We now show that the upper bound in \cref{eq:integral-upper-bound} converges to 1 for $\gamma_t^* = \frac{1}{2\sqrt{t + 2}}$.
Since $\gamma(t_c) \to 0$ as $t_1 \to \infty$, we can ignore the term $k \gamma(t_c)$ in our limit analysis. We don't want to compute $t_c$ as an integer here, so we use an approximation $t_c \leq t_{c, max} = \frac{1 + 2 \gamma(t_1)}{1 - 2 \gamma(t_1)} (t_1 + 1) + 1 = \frac{\sqrt{t_1 + 2} + 1}{\sqrt{t_1 + 2} - 1} (t_1 + 1) + 1$ for which the upper bound still holds.
We evaluate the integral:
\begin{equation*}
\begin{aligned}
	\int_{t_1}^{t_{c, max}} \gamma(t - 1) ~ dt
	&= \int_{t_1}^{t_{c, max}} \frac{1}{2 \sqrt{t + 2 - 1}} ~ dt 
	= \left[ \sqrt{t + 1} \right]_{t_1}^{t_{c, max}} \\
	&= \sqrt{\frac{\sqrt{t_1 + 2} + 1}{\sqrt{t_1 + 2} - 1} (t_1 + 1) + 2} - \sqrt{t_1 + 1}  \\
	&= \sqrt{\frac{(\sqrt{t_1 + 2} + 1)(\sqrt{t_1 + 2} + 1)}{(\sqrt{t_1 + 2} - 1)(\sqrt{t_1 + 2} + 1)} (t_1 + 1) + 2} - \sqrt{t_1 + 1} \\
	&= \sqrt{\frac{(t_1 + 2) + 2 \sqrt{t_1 + 2} + 1}{(t_1 + 2) - 1} (t_1 + 1) + 2} - \sqrt{t_1 + 1} \\
	&= \sqrt{t_1 + 2 \sqrt{t_1 + 2} + 5} - \sqrt{t_1 + 1}.
\end{aligned}
\end{equation*}

We now show that its limit is 1.
\begin{equation*}
\begin{aligned}
	\lim_{t_1 \to \infty} \int_{t_1}^{t_{c, max}} \gamma(t - 1) ~ dt
	&= \lim_{t_1 \to \infty} \sqrt{t_1 + 2 \sqrt{t_1 + 2} + 5} - \sqrt{t_1 + 1} \\
	&= \lim_{t_1 \to \infty} \left( \sqrt{t_1 + 2 \sqrt{t_1 + 2} + 5} - \sqrt{t_1 + 1} \right) \cdot \frac{\sqrt{t_1 + 2 \sqrt{t_1 + 2} + 5} + \sqrt{t_1 + 1}}{\sqrt{t_1 + 2 \sqrt{t_1 + 2} + 5} + \sqrt{t_1 + 1}} \\
	&= \lim_{t_1 \to \infty} \frac{2 \sqrt{t_1 + 2}}{\sqrt{t_1 + 2 \sqrt{t_1 + 2} + 5} + \sqrt{t_1 + 1}} + \frac{4}{\sqrt{t_1 + 2 \sqrt{t_1 + 2} + 5} + \sqrt{t_1 + 1}} \\
	&= \lim_{t_1 \to \infty} \frac{2 \sqrt{t_1 + 2}}{\sqrt{t_1 + 2 \sqrt{t_1 + 2} + 5} + \sqrt{t_1 + 1}} + 0 \\
	&= \lim_{t_1 \to \infty} \frac{2}{\sqrt{\frac{t_1 + 2 \sqrt{t_1 + 2} + 5}{t_1 + 2}} + \sqrt{\frac{t_1 + 1}{t_1 + 2}}} \\
	&= \frac{2}{1 + 1}  \\
	&= 1
\end{aligned}
\end{equation*}

\begin{figure}
	\centering
	\includegraphics{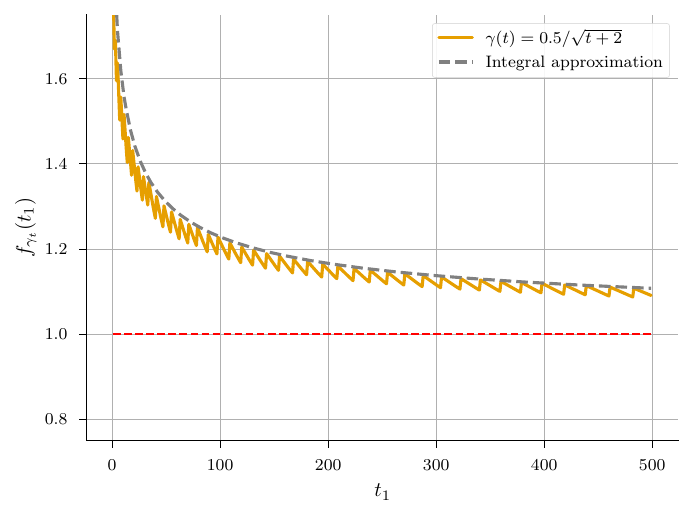}
	\caption{Visualization of the derived bound for $\gamma_t^* = \frac{1}{2 \sqrt{t + 2}}$. The orange curve displays the exact value of $\sum_{t = 1}^{t_{c, max}} \gamma(t) + k \cdot \gamma(t_c)$ with the tighter $t_{c, max} = \left\lceil \frac{1 + 2 \gamma_{t_1}}{1 - 2 \gamma_{t_1}} \cdot (t_1 + 1) \right\rceil$, while the gray line represents the integral approximation as described. Note that the numerical evaluation of the sum jumps because of the ceil operator in $t_{c, max}$.}
	\label{fig:root-harmonic-bound}
\end{figure}

We visualize the bound in \cref{fig:root-harmonic-bound}. The derived bound in itself does not justify that $\gamma_t^* = \frac{1}{2 \sqrt{t + 2}}$ is not a possible mean-based rate. In fact, our approach does not allow us to rule out this possibility. However, we can use the convergence to 1 for a study with only slightly different rates. This is what we will do next.

\subsection{Varying the rate}
\label{sec:lower-bound-variation}

Here, we first consider rates of the form
\begin{equation}
	\label{eq:mean-based-rate-template}
	\gamma_t' = \frac{\gamma_0}{(t + \tau)^{\alpha}}.
\end{equation}
For $\gamma_0 = \frac{1}{2}$, $\tau = 2$, and $\alpha = \frac{1}{2}$, we recover the rate analyzed above. We now vary these parameters to show that only slight changes can lead to an invalid mean-based rate. \cref{fig:root-harmonic-variations} visualizes a selection of these differences.

\paragraph{$\gamma_0 < 1/2$:}
It is easy to see that $f_{\gamma_t'}(t_1) < f_{\gamma_t^*}(t_1)$: First, $t_c$ under $\gamma_t'$ is at most as high as $t_c$ under $\gamma_t^*$. This means the number of steps summed by $f_{\gamma_t'}$ is at most as high as in $f_{\gamma_t^*}$. Second, every individual term in the sum is multiplied by $\gamma_0 / (1/2) < 1$, which gives the result.
As a consequence, for any $\gamma_0 < \frac{1}{2}$, $f_{\gamma_0/\sqrt{t+2}}$ converges to a value strictly smaller than one. Thus, such a sequence cannot belong to a bandit-feedback mean-based algorithm according to \cref{thm:functional-mean-based-condition}.

\paragraph{$\alpha > 1/2$:}
Consider the two sequences $\gamma_t' = \frac{1}{4 \cdot (t + 2)^{1/2}}$ and $\gamma_t'' = \frac{1}{2 (t + 2)^{\alpha}}$. $\gamma_t'$ cannot be the rate of a bandit-feedback mean-based algorithm as justified in the previous paragraph. If $\alpha > 1/2$, there will be a time $\hat{t}$ after which $\gamma''_t < \gamma'_t$ for all $t \geq \hat{t}$. Since $\gamma_t'$ is not a mean-based rate, neither is $\gamma_t''$ according to \cref{cor:mean-based-sequence-relations}.

\paragraph{$\tau > 2$:}
Increasing $\tau$ also makes the obtained rate smaller than $\gamma_t^* = \frac{1}{2 \sqrt{t + 2}}$ everywhere. However, following the steps above, we find that the upper bound remains the same. For this reason, there might as well be lower rates that still belong to a valid mean-based bandit-feedback algorithm. However, these rates cannot differ substantially from $\gamma_t^*$ in their asymptotic rate.

\begin{figure}
	\centering
	\includegraphics{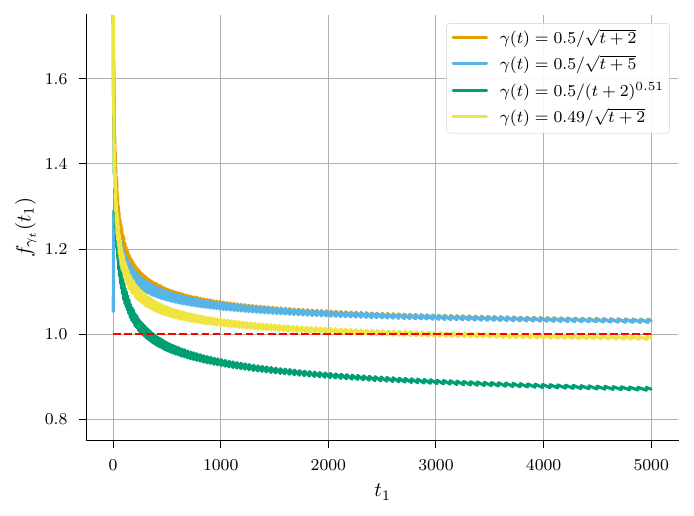}
	\caption{$f_{\gamma_t'}$ for different variants of $\gamma_t'$. We observe that, for $\gamma_0 < \frac{1}{2}$ or $\alpha > \frac{1}{2}$, $f_{\gamma_t'}$ crosses the threshold of 1.}
	\label{fig:root-harmonic-variations}
\end{figure}

Now we drop the assumption that $\gamma_t'$ has the form stated in \cref{eq:mean-based-rate-template}. If $\gamma_t' \in o(1/\sqrt{t})$, there will be a time after which $\gamma_t < 1 / (3 \sqrt{t + 2})$. By \cref{cor:mean-based-sequence-relations} and the observations above ($\gamma_0 = 1/3 < 1/2$), it thus cannot be a valid mean-based rate. 

\cref{thm:mean-based-lower-bound} summarizes these insights and establishes $\gamma_t^*$ as an "infinimum" (not literally) for the mean-based rate.

\section{Mean-based algorithms and regret}

In this section, we want to capture the relation between mean-based algorithms and regret, connecting the established view of regret with the newer mean-based perspective. 

From the literature on exploitability of mean-based algorithms (e.g. \citep{braverman_selling_2018, deng_strategizing_2019}), one might infer that the algorithms suffer regret. However, this is not necessarily the case. In \cref{sec:mean-based-exploitation-and-regret}, we show that in the example from \citet{deng_strategizing_2019}, the mean-based learner actually experiences sublinear regret. This mirrors a statement by \citet{braverman_selling_2018}, who also acknowledge that the mean-based algorithm does not experience external regret in their setting. The relation between mean-based algorithms and no-regret algorithms is thus more diverse than the exploitability results might suggest. 

Let us introduce some notation that frames the class(es) of mean-based algorithms. Let $ALG$ denote the class of all algorithms,  $MB(\gamma_t)$ be the set of $\gamma_t$-mean-based algorithms,  $MB(\gamma_t, T)$ be the set of algorithms that satisfy the mean-based definition up to time $T$, and $NR$ be the set of no-external-regret algorithms.
Clearly, $MB(\gamma_t) \subseteq MB(\gamma_t, T_2) \subseteq MB(\gamma_t, T_1)$ for any $0 < T_1 < T_2$. Similarly, if $\gamma_t' > \gamma_t$ point-wise, then $MB(\gamma_t) \subseteq MB(\gamma_t')$.
If $\gamma_t \geq 1$ until time $T$, then $MB(\gamma_t, T) = ALG$. For high $\gamma_t$ rates, the mean-based condition is rarely enforced, so mean-based algorithms can, in principle, cover many different algorithms. However, $MB(\gamma_t)$ clearly does not contain all algorithms: For example, an algorithm that always chooses one fixed action is not mean-based once $\gamma_t < 1$. 

When discussing the regret of mean-based algorithms, we aim to find lower and upper bound relations following the schemas
\begin{align}
	&& \forall ~ Alg \in MB(\gamma_t) ~ && \forall ~ (x_t)_{t = 1}^{\infty}: &&  R_T(a) &\in \mathcal{O}(f_{\gamma_t}(T)), \tag{upper bound} \label{eq:reg-upper-bound} \\
	\text{and} && \forall ~ Alg \in MB(\gamma_t) ~ && \exists ~ (x_t)_{t = 1}^{\infty}: &&  R_T(a) &\in \Omega(g_{\gamma_t}(T)), \tag{lower bound} \label{eq:reg-lower-bound}
\end{align}
where $f_{\gamma_t}$ and $g_{\gamma_t}$ represent the functional relationship between the mean-based rate $\gamma_t$ and the regret $R_T(a)$. The first schema would state that any $\gamma_t$-mean-based algorithm would suffer some limited regret, while the second one would imply that any $\gamma_t$-mean-based algorithm would suffer at least some level of regret.

Below, we will reason that we cannot derive any meaningful (sub-linear) upper bounds without additional knowledge about the algorithm. We will also argue that, aside of some special cases, deriving meaningful lower bounds (that exceed the well-known online learning bound of $\Omega(1/\sqrt{t})$) is a challenging problem.

\subsection{No meaningful upper bounds}

We first turn attention to upper bounds. Above, we already reasoned that for large $\gamma_t$, $MB(\gamma_t)$ is large, which means there could be algorithms in it that suffer from high regret. 
Here, we argue that, even for relatively small $\gamma_t$, we can show that mean-based algorithms may incur linear regret rates. Fix an arbitrary constant $c \in (0, 1/2]$ and consider a rate $\gamma_t$ such that $\gamma_t \geq c/t$ for $t \geq t_0$. We construct the following reward sequences (represented by tuples):
\begin{equation}
\label{eq:alternating_c_rewards}
	x_t(a) = \begin{cases}
		(c, 0, 2c, 0, 2c, \dots) & a = 1 \\
		(0, 2c, 0, 2c, 0, \dots) & a = 2.
	\end{cases}
\end{equation}
We observe that $\abs{\Delta \sigma(t)} = c \leq \gamma_t \cdot t$, so the mean-based property is never triggered. Consider the (full-feedback) algorithm that chooses actions with probabilities
\begin{equation*}
	p_{t+1}(a) = \begin{cases}
		\gamma_t & \exists a': ~ \sigma_t(a') - \sigma_t(a) > \gamma_t \cdot t \\
		1 & t \text{ odd and } \forall a': ~ \sigma_t(a') - \sigma_t(a) \leq \gamma_t \cdot t \\
		0 & \text{else}.
	\end{cases}
\end{equation*}
The algorithm is $\gamma_t$-mean-based by construction. In the environment from \cref{eq:alternating_c_rewards}, it picks the wrong action with probability $1$ once $\gamma_t \cdot t \geq c$, accumulating linear regret. We visualize this setup in \cref{fig:alternating_c_rewards} and summarize the finding in the proposition below.

\begin{figure}
	\centering
	\includegraphics[width=10cm]{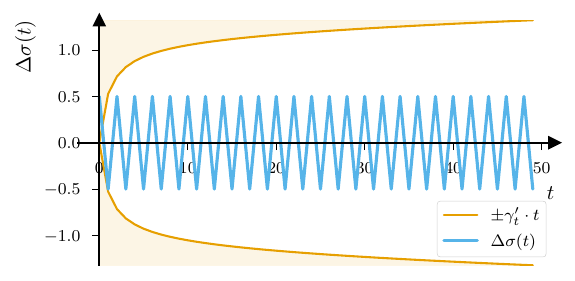}
	\caption{Mean-based algorithms in the environment defined by the rewards in \cref{eq:alternating_c_rewards}. $\Delta \sigma(t) = \sigma_t(1) - \sigma_t(2)$ is the difference in cumulative reward between action $1$ and action $2$. The mean-based property is never relevant for large $t$, which leaves room for arbitrary behavior of the algorithm, including suboptimal choices.}
	\label{fig:alternating_c_rewards}
\end{figure}

\begin{proposition}
	Let $\gamma_t \geq c/t$ for $t \geq t_0$ and $0 < c \leq 1/2$. Then there exists an algorithm $Alg \in MB(\gamma_t)$, an environment $(x_t)_{t = 0}^{\infty}$, and a constant $C \in \R$ such that $R_T(a) = C + c \cdot (T - t_0)$ for $T \geq t_0$.
\end{proposition}

Note how this rules out any meaningful sub-linear upper bound for regret of mean-based algorithms. However, it does not state that all mean-based algorithms must suffer linear regret. In the next section, we will argue that deriving a general lower bound for regret is challenging.

\subsection{Lower bounds}

As above, we cannot make any statement about a mean-based algorithm's choices when $\sigma_t(a') - \sigma_t(a) \leq \gamma_t \cdot t$ for all pairs $a, a' \in \mathcal{A}$. For a lower bound on regret, we hence need to focus our attention on reward sequences that actually trigger the threshold. That means that there exists a set of actions $\mathcal{A}_{dom} \subset \mathcal{A}$ such that $\forall a \in \mathcal{A}_{dom}: \sigma_t(a') - \sigma_t(a) > \gamma_t \cdot t$ for some $a' \in \mathcal{A}$. In order for the algorithm to accumulate regret, these actions need to be rewarded from time to time, which works against their dominated status. A suitable environment may depend on the algorithm, which makes a lower bound across all algorithms non-trivial. We acknowledge that the implications of such a statement are interesting in their own right but leave this problem open for future investigations.

We still want to provide an extreme case where the mean-based condition is always active, and which shows that mean-based algorithms may suffer from linear regret.
Consider any $1/((2 + \delta) t)$-mean-based algorithm, $\delta > 0$ and an adversarial environment with two actions and alternating reward sequences (represented by tuples)
\begin{equation}
	\label{eq:alternating_reward_sequence}
	x_t(a) = \begin{cases}
		(\frac{1}{2}, 0, 1, 0, 1, 0, \dots) & a = 1 \\
		(0, 1, 0, 1, 0, 1,  \dots) & a = 2 .
	\end{cases}
\end{equation}

We visualize the setting in \cref{fig:regret_by_mean_based_condition}. It is easy to see that the algorithm would play action $a = 1$ at even time steps and $a = 2$ at odd time steps with high probability: The mean-based condition is always active for alternating actions, making the algorithm prefer the bad action at every time step. This will necessarily lead to a high expected regret that grows linearly in $T$. 

\begin{figure}
	\centering
	\includegraphics[width=10cm]{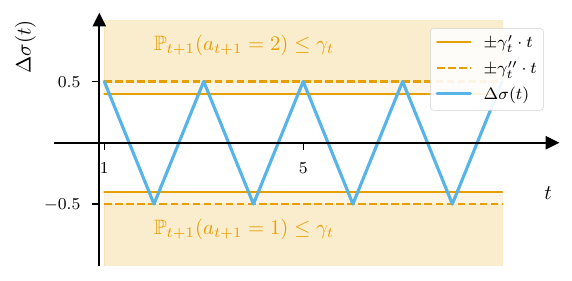}
	\caption{Mean-based algorithms in the environment defined by the rewards in \cref{eq:alternating_reward_sequence}. $\Delta \sigma(t) = \sigma_t(1) - \sigma_t(2)$ is the difference in cumulative reward between action $1$ and action $2$. With $\gamma'_t = 1/((2 + \delta) t)$, the mean-based condition is always active. while it is never active with $\gamma''_t = 1/(2t)$.}
	\label{fig:regret_by_mean_based_condition}
\end{figure}

A stronger statement can be made for internal and swap regret:
The results by \citet{braverman_selling_2018} and \citet{lin_generalized_2025} suggests that mean-based algorithms experience \textit{internal} regret. As a result, mean-based algorithms seem to be incompatible with some stronger notions of regret.

\subsection{Are mean-based algorithms bad?}

The results so far seem to suggest that mean-based algorithms may have poor no-regret properties. Combined with their exploitability in principal-agent games, this raises the question if these algorithms are inherently "bad". 

We argue that this is not the case. First, many existing algorithms are mean-based if configured accordingly. The property and its analysis helps us understand learning and its limitations better. Second, exploitability comes from a more precise description of what mean-based algorithms will do in future steps. Knowledge about an algorithm of course provides an advantage for optimizing against it.

Lastly, we reason that mean-based algorithms can be no-external-regret, even if only bandit feedback is available. 

\begin{theorem}
	The Exp3 algorithm introduced in \cref{sec:mean-based-Exp3} with bandit feedback, $\eta_t = t^{-1/3}$, and $\epsilon_t = t^{-1/4}$ is no-external-regret and mean-based.
\end{theorem}

The mean-based part of the theorem follows directly from \cref{prop:mean-based-exp3}. The no-regret part follows from common regret bounds \citep{lattimore_bandit_2020, shalev-shwartz_online_2011}. For completeness, we provide a proof outline for the regret bound in \cref{sec:proof-regret-bound-Exp3}. Note that the parameterization of $\eta_t$ and $\epsilon_t$ is not the only choice to show the two properties, which means that the overlap between no-external-regret and mean-based algorithms is not limited to a singular example.

In summary, these arguments show that mean-based algorithms \textit{can} be no-external-regret even under bandit feedback, but there is no universal guarantee of this. 
Our arguments rely on examples and extreme cases. While they establish a first showcase of the diverse relation between regret and $\gamma_t$, they do not yet give tight quantitative bounds.

\section{Numerical experiments}

This section empirically compares our mean-based algorithms with established no-regret learning algorithms. Based on this, we want to get a first impression of the performance cost of the mean-based property.

To validate our algorithms introduced in \cref{sec:mean-based-algorithms-with-bandit-feedback}, we apply them to a simple stochastic MAB setting, which can be understood as a benchmarking environment for comparing performance with established bandit algorithms. To cover non-stationary settings, we also provide experiments in a repeated Bertrand oligopoly with multiple learners, recreating a setup by \citet{durmann_online_2026}. With these simple experiments, we first and foremost want to analyze the learning speed of our algorithms.

We configure two variants of Exp3 and two variants of the $\epsilon$-Greedy bandit as described in \cref{tab:algo-config}. One variant of each algorithm achieves the balanced $\gamma_t$ rates in \cref{prop:mean-based-bandit-algorithm} and \cref{prop:mean-based-exp3}. The variants labeled with "maybe" are configured with $\epsilon_t \sim \sqrt{\log(t)/t}$, which is just outside of where \cref{prop:mean-based-bandit-algorithm} and \cref{prop:mean-based-exp3} can guarantee the mean-based property. This is also a common choice to balance regret terms.

\begin{table}[h]
	\centering
	\begin{tabular}{l l l l}
		\toprule
		Algorithm &  & $\epsilon_t$ (and $\eta_t$) & mean-based? \\
		\midrule
		Exp3 & $\alpha = 1/4$ & $\min\{1, \sqrt[4]{\log(t + 1) / (t + 1)}\}$ & yes \\
		Exp3 & $\alpha = 1/2$ & $\min\{1, \sqrt[2]{\log(t + 1) / (t + 1)}\}$ & maybe \\
		$\epsilon$-Greedy & $\alpha = 1/4$ & $\min\{1, \sqrt[4]{\log(t + 1) / (t + 1)}\}$ & yes \\
		$\epsilon$-Greedy & $\alpha = 1/2$ & $\min\{1, \sqrt[2]{\log(t + 1) / (t + 1)}\}$ & maybe \\
		\bottomrule
	\end{tabular}
	\caption{
		\centering
		Configuration of the algorithms. The $\alpha$ values indicate the order of the root.
	}
	\label{tab:algo-config}
\end{table}

\subsection{Experiments in a stochastic environment}

We consider a stochastic MAB setting with $k=10$ actions and Bernoulli-distributed rewards: Action $i$ yields reward $1$ with probability $p_i$ and gives reward $0$ with probability $1 - p_i$. The probability for each action $i$ is set by $p_i=\frac{i-1}{k-1}$, i.e., action $i=10$ has a probability of $1$ to return a reward of $1$.

We compare the empirical rate of convergence towards action 10. In addition, we evaluate the average achieved reward per round. The development of both metrics is visualized over 10,000 iterations. All reported results are averaged over 10 independent simulations.

\begin{figure}[h]
	\centering
	
	\hfill
	\begin{subfigure}[t]{6.5cm}
		\centering
		\includegraphics[width=6.5cm]{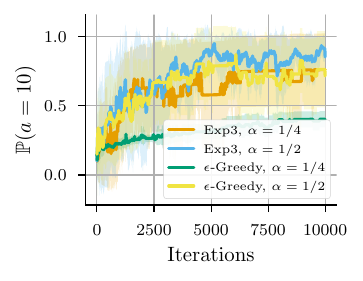}
		\caption{Probability of selecting the optimal action.}
		\label{subfig:sto-MAB-probs}
	\end{subfigure}
	\hfill
	\begin{subfigure}[t]{6.5cm}
		\centering
		\includegraphics[width=6.5cm]{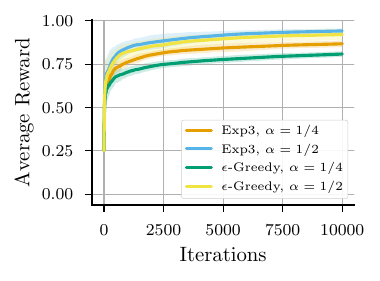}
		\caption{Average rewards.}
		\label{subfig:sto-MAB-avg-rew}
	\end{subfigure}
	\hfill
	\caption{Results in the MAB setting. Displayed are averages and standard deviations over ten runs.}
\end{figure}

As \cref{subfig:sto-MAB-probs} and \cref{subfig:sto-MAB-avg-rew} show, Exp3 with $\alpha = 1/2$ achieves the best last-iterate convergence towards the optimal action and average per-round reward. In comparison, its mean-based counterpart shows a slightly slower increase in the probability of selecting the optimal action and in the average realized reward. These observations also translate to average rewards.
The Greedy algorithms behave similarly. However, configured with $\alpha = 1/4$, the algorithm exploits the best action the least. This can be attributed to its definition, which weighs all non-dominated actions equally. Only as the number of dominated actions increases does the probability of playing the best action increase.
This behavior suggests that a rapidly decreasing exploration rate in the $\epsilon$-greedy algorithm leads to better outcomes, particularly in this unambiguous environment. We conclude that the mean-based property has a noticeable but (sometimes) small performance impact.

\subsection{Experiments in a Bertrand oligopoly}
\label{sec:experiments-bertrand}

Similar to \citet{durmann_online_2026}, we analyze a repeated Bertrand oligopoly and the convergence of our mean-based algorithms to the Nash equilibrium. 

Bertrand oligopolies model markets under pricing competition and can be stated as normal-form games. Let $n \geq 2$ denote the number of players. At time $t$, each player selects an action $a_t^i$, a "price", from a discrete action space $\mathcal{A}^i$, giving rise to a strategy profile $a_t = (a_t^1, \dots, a_t^n)$. As is common in the literature on game theory, we denote the opponents' actions by $a_t^{-i}$. Player $i$ then receives a reward
\begin{equation*}
	x^i_t(a_t^i, a_t^{-i}) = d^i(a_t^i, a_t^{-i}) \cdot (a_t^i - c^i),
\end{equation*}
where $c^i$ is the player's marginal cost, and the demand $d^i$ is computed as
\begin{equation*}
	d^i(a_t^i, a_t^{-i}) = \begin{cases}
		\begin{aligned}
			& D / {n_{min}} \cdot (1 - a_t^i / k) && \text{if $i \in \arg\min_{j \in [n]} a_t^j$} && \\
			& 0 && \text{else} && .
		\end{aligned}
	\end{cases}
\end{equation*}
Here, $D > 0$ is the maximum market demand and $n_{min} = \abs{\arg\min_{j \in [n]} a_t^j}$ is the number of firms offering the lowest price. 

In our experimental setup, we set the action space to $\mathcal{A}^i = \{0, \delta, 2 \delta, \dots,  1.0\}$ with $\delta = 0.1$, resulting in $k = 11$ available price levels. We set the marginal cost to $c = 0.1$ for both players. With $c = \delta$, the game possesses two (pure) Nash equilibria at $a^*_1 = \delta$ and $a^*_2 = 2 \delta$.

We evaluate how far the learned strategies are from the Nash equilibrium using the Wasserstein distance. We set the target Nash equilibrium to $a^*_2$, since it yields positive payoffs for the players. Let $p_t^i$ be the probability distribution that player $i$ selects at time $t$. Then  the Wasserstein distance between $p_t^i$ and the pure Nash equilibrium $a_2^*$, represented by its \textit{mixed} strategy formulation $p^*$, is given by
\begin{equation*}
	\sum_{a = 1}^{k-1} \abs{P^*(a) - P_t^i(a)} \cdot \left( \frac{a+1}{k} - \frac{a}{k} \right) = \frac{1}{k} \sum_{a = 1}^{k-1} \abs{P^*(a) - P_t^i(a)},
\end{equation*}
where $P^*(a)$ and $P_t^i(a)$ denote the cumulative distributions evaluated at $a$. 

\begin{figure}[h]
	\centering
	\includegraphics{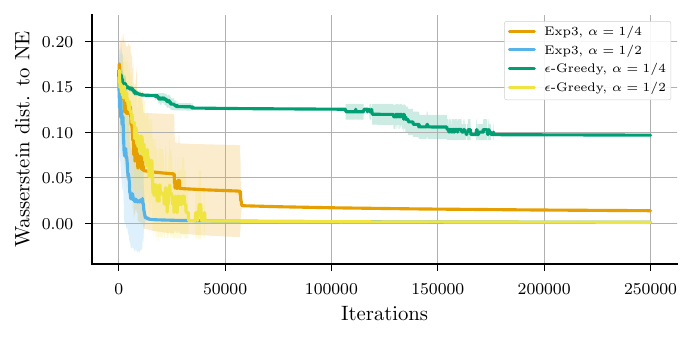}
	\caption{Convergence behavior of our algorithms in a Bertrand oligopoly. Displayed are average and standard deviation of the Wasserstein distance over ten runs.}
	\label{fig:convergence-bertrand-oligopoly}
\end{figure}

\Cref{fig:convergence-bertrand-oligopoly} visualizes the convergence rate of our algorithms in this repeated game. The observations are aligned with those in the MAB scenario: Both Exp3 algorithms perform reasonably well and show quick and consistent convergence. The $\epsilon$-greedy algorithms, while converging, are noticeably slower, and with $\alpha = 1/4$, we find a substantial gap to the other algorithms. Again, the algorithm is probably bound to excessive exploration because its domination threshold is large. What stands out is the speed of Exp3 with $\alpha = 1/4$: In contrast to the experiments by \citet{durmann_online_2026}, our algorithm is not only guaranteed to converge \textit{eventually} - it also does so in reasonable time.

\FloatBarrier

\section{Conclusion}

In this paper, we have studied mean-based algorithms under bandit feedback with an unknown horizon. Mean-based algorithms offer a different perspective on online learning problems, particularly the MAB problem. 

Our main result provides the first lower bound on the $\gamma_t$ rate that these algorithms can achieve under bandit feedback. Our proof technique frames the exploration-exploitation dilemma of MAB problems in the context of mean-based algorithms: if the algorithm explores insufficiently, it will not be able to identify when it is repeatedly playing the wrong action. 

This result has implications for multiarm-bandit learning in general, as it limits how quickly any algorithm can reliably distinguish between low- and high-reward actions while playing accordingly. In the context of multi-agent learning, our bound limits how quickly independent learners can identify the set of correlated rationalizable actions in normal-form games.

Our work also proposes two algorithms that extend the well-known Exp3 and $\epsilon$-greedy algorithms. 
Their derived rates $\gamma_t \in \mathcal{O}(\sqrt[4]{\log t / t})$ do not yet match our lower bound of $\gamma_t \in \Omega(1/\sqrt{t})$. This leaves room for future investigations.

Our experiments suggest that while our mean-based algorithms are generally slower, the mean-based version of Exp3 can compete with other Exp3 or $\epsilon$-greedy variants in terms of convergence speed and realized reward. This contrasts with the previous findings of \citet{durmann_online_2026} and shows that mean-based algorithms are an alternative worth considering. 

We have also discussed the diverse relation between mean-based algorithms and regret. For now, mean-basedeness and no-regret appear to be mostly independent concepts that don't imply or preclude each other. While an upper bound on regret cannot be derived without further information about the algorithm, a lower bound might be an interesting avenue for future work. Our Exp3 algorithm demonstrates that no-regret and mean-based properties can be achieved simultaneously, even under bandit feedback. This leads to algorithms that implement intuitive ideas while maintaining desirable performance guarantees. 

\section*{Acknowledgements}

We would like to thank Jan Hoehener for helping us derive the limit of the integral bound. 

This project has received funding from the European Research Council (ERC) under the European Union’s Horizon Europe research and innovation programme (grant agreement No 101198689).

This project was funded by the Deutsche Forschungsgemeinschaft (DFG, German Research Foundation) - GRK 2201/2 - Project Number 277991500 and BI 1057/9.

\vfill
\pagebreak

\bibliographystyle{plainnat}
\bibliography{references}

@book{lattimore_bandit_2020,
	address = {Cambridge},
	edition = {1},
	title = {Bandit {Algorithms}},
	isbn = {978-1-108-48682-8},
	url = {https://www.cambridge.org/core/books/bandit-algorithms/8E39FD004E6CE036680F90DD0C6F09FC},
	doi = {10.1017/9781108571401},
	language = {en},
	urldate = {2023-07-03},
	publisher = {Cambridge University Press},
	author = {Lattimore, Tor and Szepesvári, Csaba},
	year = {2020},
}

@inproceedings{auer_gambling_1995,
	address = {1730 Massachusetts Ave., NW Washington, DCUnited States},
	title = {Gambling in a rigged casino: {The} adversarial multi-armed bandit problem},
	issn = {0272-5428},
	shorttitle = {Gambling in a rigged casino},
	doi = {10.1109/SFCS.1995.492488},
	booktitle = {Proceedings of {IEEE} 36th {Annual} {Foundations} of {Computer} {Science}},
	publisher = {IEEE Computer Society},
	author = {Auer, P. and Cesa-Bianchi, N. and Freund, Y. and Schapire, R.E.},
	editor = {Raghavan, Prabhakar},
	month = oct,
	year = {1995},
	pages = {322--331},
}

@article{auer_finite-time_2002,
	title = {Finite-time {Analysis} of the {Multiarmed} {Bandit} {Problem}},
	volume = {47},
	issn = {1573-0565},
	url = {https://doi.org/10.1023/A:1013689704352},
	doi = {10.1023/A:1013689704352},
	language = {en},
	number = {2},
	urldate = {2023-09-01},
	journal = {Machine Learning},
	author = {Auer, Peter and Cesa-Bianchi, Nicolò and Fischer, Paul},
	month = may,
	year = {2002},
	pages = {235--256},
}

@article{thompson_likelihood_1933,
	title = {On the {Likelihood} that {One} {Unknown} {Probability} {Exceeds} {Another} in {View} of the {Evidence} of {Two} {Samples}},
	volume = {25},
	issn = {0006-3444},
	url = {https://www.jstor.org/stable/2332286},
	doi = {10.2307/2332286},
	number = {3/4},
	urldate = {2023-09-28},
	journal = {Biometrika},
	publisher = {[Oxford University Press, Biometrika Trust]},
	author = {Thompson, William R.},
	year = {1933},
	pages = {285--294},
}

@inproceedings{deng_nash_2022,
	address = {New York, NY, USA},
	series = {{WWW} '22},
	title = {Nash {Convergence} of {Mean}-{Based} {Learning} {Algorithms} in {First} {Price} {Auctions}},
	isbn = {978-1-4503-9096-5},
	url = {https://dl.acm.org/doi/10.1145/3485447.3512059},
	doi = {10.1145/3485447.3512059},
	urldate = {2023-10-17},
	booktitle = {Proceedings of the {ACM} {Web} {Conference} 2022},
	publisher = {Association for Computing Machinery},
	author = {Deng, Xiaotie and Hu, Xinyan and Lin, Tao and Zheng, Weiqiang},
	editor = {{Frédérique Laforest} and {Raphaël Troncy} and {Elena Simperl} and {Deepak Agarwal} and {Aristides Gionis} and {Ivan Herman} and {Lionel Médini}},
	month = apr,
	year = {2022},
	pages = {141--150},
}

@inproceedings{kolumbus_auctions_2022,
	address = {New York, NY, USA},
	series = {{WWW} '22},
	title = {Auctions between {Regret}-{Minimizing} {Agents}},
	isbn = {978-1-4503-9096-5},
	url = {https://dl.acm.org/doi/10.1145/3485447.3512055},
	doi = {10.1145/3485447.3512055},
	urldate = {2023-10-17},
	booktitle = {Proceedings of the {ACM} {Web} {Conference} 2022},
	publisher = {Association for Computing Machinery},
	author = {Kolumbus, Yoav and Nisan, Noam},
	editor = {{Frédérique Laforest} and {Raphaël Troncy} and {Elena Simperl} and {Deepak Agarwal} and {Aristides Gionis} and {Ivan Herman} and {Lionel Médini}},
	month = apr,
	year = {2022},
	pages = {100--111},
}

@inproceedings{braverman_selling_2018,
	address = {New York, NY, USA},
	series = {{EC} '18},
	title = {Selling to a {No}-{Regret} {Buyer}},
	isbn = {978-1-4503-5829-3},
	url = {https://dl.acm.org/doi/10.1145/3219166.3219233},
	doi = {10.1145/3219166.3219233},
	urldate = {2024-01-23},
	booktitle = {Proceedings of the 2018 {ACM} {Conference} on {Economics} and {Computation}},
	publisher = {Association for Computing Machinery},
	author = {Braverman, Mark and Mao, Jieming and Schneider, Jon and Weinberg, Matt},
	editor = {Tardos, Eva and Elkind, Edith and Vohra, Rakesh},
	month = jun,
	year = {2018},
	pages = {523--538},
}

@article{shalev-shwartz_online_2011,
	title = {Online {Learning} and {Online} {Convex} {Optimization}},
	volume = {4},
	issn = {1935-8237, 1935-8245},
	url = {http://www.nowpublishers.com/article/Details/MAL-018},
	doi = {10.1561/2200000018},
	language = {en},
	number = {2},
	urldate = {2024-01-25},
	journal = {Foundations and Trends® in Machine Learning},
	author = {Shalev-Shwartz, Shai},
	year = {2011},
	pages = {107--194},
}

@article{brandenburger_rationalizability_1987,
	title = {Rationalizability and {Correlated} {Equilibria}},
	volume = {55},
	issn = {0012-9682},
	url = {https://www.jstor.org/stable/1913562},
	doi = {10.2307/1913562},
	number = {6},
	urldate = {2024-09-11},
	journal = {Econometrica},
	publisher = {Wiley, The Econometric Society},
	author = {Brandenburger, Adam and Dekel, Eddie},
	year = {1987},
	pages = {1391--1402},
}

@misc{arunachaleswaran_algorithmic_2024,
	title = {Algorithmic {Collusion} {Without} {Threats}},
	url = {http://arxiv.org/abs/2409.03956},
	doi = {10.48550/arXiv.2409.03956},
	urldate = {2025-04-22},
	publisher = {arXiv},
	author = {Arunachaleswaran, Eshwar Ram and Collina, Natalie and Kannan, Sampath and Roth, Aaron and Ziani, Juba},
	month = dec,
	year = {2024},
	note = {arXiv:2409.03956 [cs]},
}

@inproceedings{hartline_regulation_2025,
	address = {New York, NY, USA},
	series = {{CSLAW} '25},
	title = {Regulation of {Algorithmic} {Collusion}, {Refined}: {Testing} {Pessimistic} {Calibrated} {Regret}},
	isbn = {979-8-4007-1421-4},
	shorttitle = {Regulation of {Algorithmic} {Collusion}, {Refined}},
	url = {https://dl.acm.org/doi/10.1145/3709025.3712217},
	doi = {10.1145/3709025.3712217},
	urldate = {2025-07-28},
	booktitle = {Proceedings of the 2025 {Symposium} on {Computer} {Science} and {Law}},
	publisher = {Association for Computing Machinery},
	author = {Hartline, Jason D. and Wang, Chang and Zhang, Chenhao},
	year = {2025},
	pages = {108--120},
}

@inproceedings{deng_strategizing_2019,
	title = {Strategizing against {No}-regret {Learners}},
	volume = {32},
	url = {https://proceedings.neurips.cc/paper_files/paper/2019/file/8b6dd7db9af49e67306feb59a8bdc52c-Paper.pdf},
	booktitle = {Advances in {Neural} {Information} {Processing} {Systems}},
	publisher = {Curran Associates, Inc.},
	author = {Deng, Yuan and Schneider, Jon and Sivan, Balasubramanian},
	editor = {Wallach, H. and Larochelle, H. and Beygelzimer, A. and Alché-Buc, F. d' and Fox, E. and Garnett, R.},
	year = {2019},
}

@misc{lin_generalized_2025,
	title = {Generalized {Principal}-{Agent} {Problem} with a {Learning} {Agent}},
	url = {http://arxiv.org/abs/2402.09721},
	doi = {10.48550/arXiv.2402.09721},
	urldate = {2025-11-18},
	publisher = {arXiv},
	author = {Lin, Tao and Chen, Yiling},
	month = oct,
	year = {2025},
	note = {arXiv:2402.09721 [cs]},
}

@article{blum_external_2007,
	title = {From {External} to {Internal} {Regret}},
	volume = {8},
	url = {http://jmlr.org/papers/v8/blum07a.html},
	language = {en},
	number = {47},
	journal = {Journal of Machine Learning Research},
	author = {Blum, Avrim and Mansour, Yishay},
	year = {2007},
	pages = {1307--1324},
}

@inproceedings{audibert_minimax_2009,
	address = {Montreal, Canada},
	title = {Minimax policies for adversarial and stochastic bandits},
	url = {https://enpc.hal.science/hal-00834882},
	booktitle = {Proceedings of the 22th annual conference on learning theory},
	author = {Audibert, Jean-Yves and Bubeck, Sébastien},
	month = jun,
	year = {2009},
	pages = {217--226},
}

@article{kolumbus_contracting_2024,
	title = {Contracting with a {Learning} {Agent}},
	volume = {37},
	url = {https://proceedings.neurips.cc/paper_files/paper/2024/hash/8d7c8a3a0ed04006d129b3cebcac7a3e-Abstract-Conference.html},
	doi = {10.52202/079017-2460},
	language = {en},
	urldate = {2026-01-07},
	journal = {Advances in Neural Information Processing Systems},
	author = {Kolumbus, Yoav and Schneider, Jon and Talgam-Cohen, Inbal and Vlatakis-Gkaragkounis, Emmanouil-Vasileios and Wang, Joshua R. and Weinberg, S. M.},
	month = dec,
	year = {2024},
	pages = {77366--77408},
}

@article{freund_decision-theoretic_1997,
	title = {A {Decision}-{Theoretic} {Generalization} of {On}-{Line} {Learning} and an {Application} to {Boosting}},
	volume = {55},
	issn = {0022-0000},
	url = {https://www.sciencedirect.com/science/article/pii/S002200009791504X},
	doi = {10.1006/jcss.1997.1504},
	number = {1},
	urldate = {2026-02-08},
	journal = {Journal of Computer and System Sciences},
	author = {Freund, Yoav and Schapire, Robert E},
	month = aug,
	year = {1997},
	pages = {119--139},
}

@article{littlestone_weighted_1994,
	title = {The {Weighted} {Majority} {Algorithm}},
	volume = {108},
	issn = {0890-5401},
	url = {https://www.sciencedirect.com/science/article/pii/S0890540184710091},
	doi = {10.1006/inco.1994.1009},
	number = {2},
	urldate = {2026-02-08},
	journal = {Information and Computation},
	author = {Littlestone, N. and Warmuth, M. K.},
	month = feb,
	year = {1994},
	pages = {212--261},
}

@inproceedings{ito_tight_2020,
	title = {A {Tight} {Lower} {Bound} and {Efficient} {Reduction} for {Swap} {Regret}},
	volume = {33},
	url = {https://proceedings.neurips.cc/paper_files/paper/2020/hash/d79c8788088c2193f0244d8f1f36d2db-Abstract.html},
	urldate = {2026-04-01},
	booktitle = {Advances in {Neural} {Information} {Processing} {Systems}},
	publisher = {Curran Associates, Inc.},
	author = {Ito, Shinji},
	year = {2020},
	pages = {18550--18559},
}

@article{hart_regret-based_2003,
	series = {Special {Issue} in {Honor} of {Robert} {W}. {Rosenthal}},
	title = {Regret-based continuous-time dynamics},
	volume = {45},
	issn = {0899-8256},
	url = {https://www.sciencedirect.com/science/article/pii/S0899825603001787},
	doi = {10.1016/S0899-8256(03)00178-7},
	number = {2},
	urldate = {2026-04-13},
	journal = {Games and Economic Behavior},
	author = {Hart, Sergiu and Mas-Colell, Andreu},
	month = nov,
	year = {2003},
	pages = {375--394},
}

@inproceedings{trovo_multi-armed_2015,
	title = {Multi-armed bandit for pricing},
	booktitle = {Proceedings of the 12th {European} {Workshop} on {Reinforcement} {Learning}},
	author = {Trovo, Francesco and Paladino, Stefano and Restelli, Marcello and Gatti, Nicola},
	year = {2015},
	pages = {1--9},
}

@inproceedings{li_contextual-bandit_2010,
	title = {A {Contextual}-{Bandit} {Approach} to {Personalized} {News} {Article} {Recommendation}},
	url = {http://arxiv.org/abs/1003.0146},
	doi = {10.1145/1772690.1772758},
	urldate = {2026-08-21},
	booktitle = {Proceedings of the 19th international conference on {World} wide web},
	author = {Li, Lihong and Chu, Wei and Langford, John and Schapire, Robert E.},
	month = apr,
	year = {2010},
	note = {arXiv:1003.0146 [cs.LG]},
	pages = {661--670},
}

@inproceedings{appavoo_shrewd_2018,
	title = {Shrewd {Selection} {Speeds} {Surfing}: {Use} {Smart} {EXP3}!},
	isbn = {2575-8411},
	url = {https://ieeexplore.ieee.org/abstract/document/8416291},
	doi = {10.1109/ICDCS.2018.00028},
	booktitle = {2018 {IEEE} 38th {International} {Conference} on {Distributed} {Computing} {Systems} ({ICDCS})},
	publisher = {IEEE},
	author = {Appavoo, A. Meetoo and Gilbert, S. and Tan, K. -L.},
	month = jul,
	year = {2018},
	note = {Journal Abbreviation: 2018 IEEE 38th International Conference on Distributed Computing Systems (ICDCS)},
	pages = {188--199},
}

@inproceedings{croissant_real-time_2020,
	title = {Real-{Time} {Optimisation} for {Online} {Learning} in {Auctions}},
	issn = {2640-3498},
	url = {https://proceedings.mlr.press/v119/croissant20a.html},
	language = {en},
	urldate = {2026-08-21},
	booktitle = {Proceedings of the 37th {International} {Conference} on {Machine} {Learning}},
	publisher = {PMLR},
	author = {Croissant, Lorenzo and Abeille, Marc and Calauzenes, Clement},
	month = nov,
	year = {2020},
	pages = {2217--2226},
}

@article{durmann_online_2026,
	title = {Online {Optimization} {Algorithms} in {Repeated} {Price} {Competition}: {Equilibrium} {Learning} and {Algorithmic} {Collusion}},
	issn = {1523-4614, 1526-5498},
	shorttitle = {Online {Optimization} {Algorithms} in {Repeated} {Price} {Competition}},
	url = {https://pubsonline.informs.org/doi/10.1287/msom.2024.1389},
	doi = {10.1287/msom.2024.1389},
	language = {en},
	urldate = {2026-09-03},
	journal = {Manufacturing \& Service Operations Management},
	publisher = {INFORMS},
	author = {Durmann, Julius and Oberlechner, Matthias and Bichler, Martin},
	month = aug,
	year = {2026},
}
\addcontentsline{toc}{section}{References}

\vfill
\pagebreak

\appendix

\vfill
\pagebreak

\appendix
\crefalias{section}{appendix}
\crefalias{subsection}{appendix}

\section{Additional proofs}
\label{sec:proofs}

\subsection{Conversion between predictive-mean-based and non-predictive-mean-based}
\label{sec:proof-predictive-equals-non-predictive-mean-based}

\begin{proof}[Proof of \cref{proposition:equivalence-of-mean-based-definitions}] 
	\leavevmode
	
	\begin{enumerate}
		\item Assume the algorithm is \textit{predictive} $\gamma_t$-mean-based. Let $a, a'$ be a pair such that $\Delta \sigma_t > t \gamma_t'$.
		We need to show that the algorithm plays $a$ with probability of at most $\gamma_t'$. Recall that the rewards are bounded: $x_{t + 1}(a), x_{t + 1}(a') \in [0, 1]$.
		
		\begin{equation*}
		\begin{aligned}
			\Delta \sigma_{t + 1} 
			&\geq \Delta \sigma_t - 1 > t \gamma_t' - 1 \\
			&= (t + 1) (\gamma_t' \cdot (1 - \frac{1}{t + 1}) - \frac{1}{t + 1}) \\
			&= (t + 1) \cdot \gamma_t
		\end{aligned}
		\end{equation*}
		
		Since the algorithm is predictive $\gamma_t$-mean-based, we get $\mathbb{P}_{t + 1}(a_{t + 1} = a) < \gamma_t \leq \gamma_t'$. (It is trivial to see that $\gamma_t' \geq \gamma_t$.)
		
		We need to ensure that the new rate $\gamma_t'$ is still monotonically decreasing:
		
		\begin{equation*}
		\begin{aligned}
			\gamma_{t + 1}' - \gamma_t'
			&= \frac{\gamma_{t + 1} \cdot (t + 2) + 1}{t + 1} - \frac{\gamma_t \cdot (t + 1) + 1}{t} \\
			&= \frac{(t + 2) \cdot t \cdot \gamma_{t + 1} + t - (t + 1)^2 \cdot \gamma_t - (t + 1) }{t \cdot (t + 1)}  \\
			&\leq \frac{(t^2 + 2t) \cdot \gamma_{t} - (t^2 + 2t + 1) \cdot \gamma_t - 1}{t \cdot (t + 1)}  \\
			&\leq \frac{-\gamma_t - 1}{t \cdot (t + 1)}  \\
			&\leq 0
		\end{aligned}
		\end{equation*}
		
		By definition, $\gamma_t'$ is always larger than zero. Additionally, as $t \to \infty$, it will approach $\gamma_t$ and thus also converge to zero.
		
		\item 
		Assume the algorithm is (non-predictive) $\gamma_t$-mean-based. 
		Let $a', a$ be a pair such that $\Delta \sigma_{t+1} > (t + 1) \gamma_t'$.
		We need to show that the algorithm plays $a$ with probability of at most $\gamma_t'$.
		
		\begin{equation*}
		\begin{aligned}
			\Delta \sigma_t 
			&\geq \Delta \sigma_{t + 1} - 1 
			> (t + 1) \gamma_t' - 1 \\
			&= t (\gamma_t' + \frac{1}{t} \gamma_t' - \frac{1}{t}) \\
			&= t \gamma_t
		\end{aligned}
		\end{equation*}
		
		Since the algorithm is $\gamma_t$-mean-based, we get $\mathbb{P}_{t + 1}(a_{t + 1} = a) < \gamma_t \leq \gamma_t'$. The inequality $\gamma_t \leq \gamma_t'$ follows since $\gamma_t \leq 1$:
		
		\begin{multline*}
			\gamma_t' \geq \gamma_t \iff \frac{\gamma_t + 1/t}{1 + 1/t} \geq \gamma_t 
			\iff \gamma_t + 1/t  \geq \gamma_t + \gamma_t / t \iff 1 \geq \gamma_t.
		\end{multline*}
		
		We need to ensure that the new rate $\gamma_t'$ is still monotonically decreasing:
		
		\begin{equation*}
		\begin{aligned}
			\gamma_{t + 1}' - \gamma_t'
			&= \frac{\gamma_{t + 1} \cdot (t + 1) + 1}{t + 2} - \frac{\gamma_t \cdot t + 1}{t + 1} \\
			&= \frac{(t + 1)^2 \gamma_{t + 1} + (t + 1) - (t + 2) t \gamma_t - (t + 2)}{(t + 1) (t + 2)}  \\
			&\leq \frac{(t^2 + 2 t + 1) \gamma_{t} - (t^2 + 2t) \gamma_t - 1}{(t + 1) (t + 2)}  \\
			&\leq \frac{ \gamma_{t} - 1}{(t + 1) (t + 2)}  \\
			&\leq 0
		\end{aligned}
		\end{equation*}
		
		By definition, $\gamma_t'$ is also always larger than zero. Additionally, as $t \to \infty$, it will approach $\gamma_t$ and thus also converge to zero.
	\end{enumerate}
\end{proof}

\subsection{Mean-based rates of our algorithms}
\label{sec:proofs-mean-based-rates}

\begin{proof}[Proof of \cref{prop:mean-based-bandit-algorithm}. (c.f. \citep{braverman_selling_2018}, proof of Theorem D.3)]
	\label{proof:mean-based-bandit-algorithm}
	\leavevmode
	
	Let $a, a'$ be such that $\sigma_t(a') - \sigma_t(a) > \gamma_t \cdot t$.
	By $E$, we denote the event that $\hat{\sigma}_t(a') - \hat{\sigma}_t(a) \leq \eta_t \cdot t$.
	
	We can decompose the probability that the algorithm chooses $a_{t + 1} = a$ based on the event $E$.
	\begin{equation*}
	\begin{aligned}
		\mathbb{P}_{t + 1}(a_{t+1} = a) &= p_{t+1}(a \vert E) \cdot \mathbb{P}(E) + p_{t+1}(a \vert \neg E) \cdot \mathbb{P}(\neg E) \\
		&\leq \underbrace{\frac{1 - n_{dom} \cdot \epsilon_t / k}{k - n_{dom}} \cdot \mathbb{P}(E)}_{(I)} + \underbrace{\frac{\epsilon_t}{k} \cdot \mathbb{P}(\neg E)}_{(II)}
	\end{aligned}
	\end{equation*}
	
	We can trivially bound $(II) \leq \epsilon_t / k$ since $\mathbb{P}(\neg E) \leq 1$.
	It remains to provide a bound on $(I)$. For this, we provide an upper bound on the probability that our algorithm's estimates $\hat{x}_t(a)$ are too far from the real values $x_t(a)$. 
	We observe that the sequence $\hat{\sigma}_t(a) - \sigma_t(a)$, $t \in \N$, is a bounded martingale:
	\begin{enumerate}
		\item Because of how we defined $\hat{x}_t(a)$, $\hat{\sigma}_t(a)$ is a conditionally unbiased estimate of $\sigma_t(a)$:
		\begin{equation*}
			\begin{aligned}
				\E[\hat{\sigma}_{t + 1}(a) - \sigma_{t + 1}(a) \mid \hat{\sigma}_t(a) - \sigma_t(a), \dots, \hat{\sigma}_1(a) - \sigma_1(a)] - (\hat{\sigma}_t(a) - \sigma_t(a)) = \hat{x}_{t + 1}(a) - x_{t + 1}(a) = 0.
			\end{aligned}
		\end{equation*}
		
		\item The distance between two steps of the sequence is bounded by $k/\epsilon_t$:
		\begin{equation*}
			\vert (\hat{\sigma}_{t}(a) - \sigma_{t}(a)) - (\hat{\sigma}_{t - 1}(a) - \sigma_{t - 1}(a)) \vert = \vert \hat{x}_{t}(a) - x_{t}(a) \vert \leq \frac{k}{\epsilon_t}
		\end{equation*}
		since $x_t \in [0, 1]$ and $p_t(a) \geq \epsilon_t / k$:
		\begin{equation*}
			\frac{x_t(a)}{p_t(a)} \cdot \mathbbm{1}[a_t = a] - x_t(a) \leq \frac{k}{\epsilon_t}, \qquad \frac{x_t(a)}{p_t(a)} \cdot \mathbbm{1}[a_t = a] - x_t(a) \geq -1 \geq - \frac{k}{\epsilon_t}.
		\end{equation*}
	\end{enumerate}
	We apply the Azuma-Hoeffding inequality and find
	\begin{equation*}
		\mathbb{P}(\vert \hat{\sigma}_t(a) - \sigma_t(a) \vert \geq M_t) \leq 2 \exp \left( - \frac{M_t^2}{2 \sum_{\tau = 1}^{t} (k / {\epsilon_\tau})^2} \right).
	\end{equation*}
	
	This lets us provide a bound on $(I)$, conditioned on our initial assumption that $\sigma_t(a') - \sigma_t(a) > \gamma_t \cdot t$. Since $1 - n_{dom} \cdot \epsilon_t / k \leq k - n_{dom}$ (because $n_{dom} \leq k - 1$), we have $(I) \leq \mathbb{P}(E)$. Let $0 < \eta_t \cdot t \leq \gamma_t \cdot t - 2 M_t$.\footnote{We will choose the sequences accordingly.} Conditioned on $\sigma_t(a') - \sigma_t(a) > \gamma_t \cdot t$, we can bound $\mathbb{P}(E)$:
	\begin{equation*}
	\begin{aligned}
		\mathbb{P} \left( E ~ \vert ~ \sigma_t(a') - \sigma_t(a) > \gamma_t \cdot t \right)
		&= \mathbb{P} \left(\hat{\sigma}_t(a') - \hat{\sigma}_t(a) \leq \eta_t \cdot t ~ \vert ~ \sigma_t(a') - \sigma_t(a) > \gamma_t \cdot t \right) \\
		&\leq \mathbb{P} \left(\hat{\sigma}_t(a') - \hat{\sigma}_t(a) \leq \gamma_t \cdot t - 2 M_t ~ \vert ~ \sigma_t(a') - \sigma_t(a) > \gamma_t \cdot t \right) \\
		&= \mathbb{P} \left( \left.
		\begin{gathered}
			(\hat{\sigma}_t(a') - \sigma_t(a')) \\
			- (\hat{\sigma}_t(a) - \sigma_t(a)) \\
			+ (\sigma_t(a') - \sigma_t(a))
		\end{gathered}
		\leq \gamma_t \cdot t - 2 M_t ~ \right\vert ~ \sigma_t(a') - \sigma_t(a) > \gamma_t \cdot t \right) \\
		&\leq \mathbb{P} \left( (\hat{\sigma}_t(a') - \sigma_t(a')) - (\hat{\sigma}_t(a) - \sigma_t(a))  \leq - 2 M_t  \right) \\
		&= \mathbb{P} \left( - (\hat{\sigma}_t(a') - \sigma_t(a')) + (\hat{\sigma}_t(a) - \sigma_t(a))  \geq 2 M_t  \right) \\
		&\leq \mathbb{P} \left( \vert\hat{\sigma}_t(a') - \sigma_t(a')\vert + \vert\hat{\sigma}_t(a) - \sigma_t(a)\vert  \geq 2 M_t  \right) \\
		&\leq \mathbb{P} \left( \vert\hat{\sigma}_t(a') - \sigma_t(a')\vert \geq M_t \vee \vert\hat{\sigma}_t(a) - \sigma_t(a)\vert  \geq M_t  \right) \\
		&\leq \mathbb{P} \left( \vert\hat{\sigma}_t(a') - \sigma_t(a')\vert \geq M_t) + \mathbb{P}(\vert\hat{\sigma}_t(a) - \sigma_t(a)\vert  \geq M_t  \right) \\
		&\leq 4 \exp \left( - \frac{M_t^2}{2 \sum (k / {\epsilon_\tau})^2} \right).
	\end{aligned}
	\end{equation*}
	In the last step, we used our bound from the Azuma-Hoeffding inequality.
	
	In summary, we now need to select sequences $\eta_t$, $\epsilon_t$, $M_t$, and $\gamma_t$ such that
	\begin{align*}
		\mathbb{P}_{t + 1}(a_{t+1} = a \mid \sigma_t(a') - \sigma_t(a) > \gamma_t \cdot t) &\leq 4 \exp \left( - \frac{M_t^2}{2 \sum k / \epsilon_\tau^2} \right) + \epsilon_t \leq \gamma_t \\
		\text{and} \qquad 0 < \eta_t \cdot t &\leq \gamma_t \cdot t - 2 M_t.
	\end{align*}
	
	To make the terms nicer, we choose $M_t = \sqrt{2 p \log(t) \cdot t} \cdot k / \epsilon_t$ and find
	\begin{equation*}
	\begin{aligned}
		\exp \left( - \frac{M_t^2}{2 \sum {(k / \epsilon_\tau)^2}} \right) 
		&\leq \exp \left( - \frac{M_t^2}{2 \sum (k / \epsilon_t)^2} \right) \\
		&= \exp \left( - \frac{M_t^2}{2 t k^2 / \epsilon_t^2} \right) \\
		&= \exp \left( - \frac{(M_t \epsilon_t / k)^2}{2 t} \right) \\
		&= \exp \left( - \frac{(\sqrt{2 p \log(t) \cdot t})^2}{2 t} \right) \\
		&= \frac{1}{t^p}.
	\end{aligned}
	\end{equation*}
	
	We then choose $\gamma_t$ to satisfy our constraints:
	\begin{equation*}
	\begin{aligned}
		\gamma_t &= \max\{\eta_t, \epsilon_t\} + 2 M_t / t + \frac{1}{t^p} \\
		&= \max\{\eta_t, \epsilon_t\} + \frac{k}{\epsilon_t} \cdot \frac{\sqrt{2 p \log(t) \cdot t}}{t} + \frac{1}{t^p} \\
		&= \max\{\eta_t, \epsilon_t\} + \frac{k \cdot \sqrt{\log(t)}}{\epsilon_t \cdot \sqrt{t}} \cdot \sqrt{2 p} + \frac{1}{t^p}
	\end{aligned}
	\end{equation*}
	
	With the choices $\eta_t \leq \epsilon_t = \min\{1, \frac{{\log(t)}^{1/4}}{t^{1/4}}\}$, we then get for sufficiently large $t$ (such that $\epsilon_t < 1$),
	\begin{equation*}
		\gamma_t = (1 + \sqrt{2p} \cdot k) \cdot \sqrt[4]{\frac{{\log(t)}}{t}} + \frac{1}{t^p}. 
		= \Omega \left( \sqrt[4]{\frac{{\log(t)}}{t}} \right).
	\end{equation*}
	
	This choice balances the terms and makes $\gamma_t$ monotonically decreasing for $t \geq 2$, which can be seen from the derivative of $\gamma_t$. We can find other (slower) rates $\gamma_t$ for any $\epsilon_t$ that decays slower than $\sqrt[4]{\frac{\log(t)}{t}}$. 
	
\end{proof}

\begin{proof}[Proof of \cref{prop:mean-based-exp3} (c.f. \citep{braverman_selling_2018}, proof of Thm. D.1, D.3)]
	\leavevmode
	
	Let $a, a' \in [k]$ be such that $\sigma_t(a') - \sigma_t(a) > \gamma_t \cdot t$.
	
	With \textbf{full feedback (MWU/Hedge)}, we have $w_t(a) = \exp(\eta_t \sigma_t(a))$, $\eta_t = \epsilon_t = 1 / \sqrt{t}$, and $\gamma_t = \frac{\log(t \cdot \epsilon_t + 5)}{t \cdot \eta_t} = \frac{\log(\sqrt{t} + 5)}{\sqrt{t}}$. The probability that the algorithm plays $a$ becomes
	\begin{equation*}
		\begin{aligned}
			p_{t + 1}(a) 
			&= (1 - \epsilon_t) \cdot  \frac{w_{t}(a)}{\sum_{\tilde{a} \in [k]} w_t(\tilde{a})}  + \frac{\epsilon_t}{k} \\
			&\leq \frac{w_t(a)}{w_t(a')} + \frac{\epsilon_t}{k} \\
			&= \exp\left( \eta_t \cdot (\sigma_t(a) - \sigma_t(a')) \right) + \frac{\epsilon_t}{k} \\
			&< \exp\left( - \eta_t \cdot \gamma_t \cdot t \right) + \frac{\epsilon_t}{k}\\
			&= \frac{1}{t \cdot \epsilon_t + 5} + \frac{\epsilon_t}{k} \\
			&\leq \frac{1}{\sqrt{t}} + \frac{1}{k \cdot \sqrt{t}} \\
			&\leq \gamma_t
		\end{aligned}
	\end{equation*}
	The derivative of $\gamma_t$ is negative for all $t > 0$.
	
	With \textbf{bandit feedback (Exp3)}, our algorithm only has access to noisy estimates $\hat{\sigma}_t$. We show that
	\begin{equation*}
		\mathbb{P}_{t+1}(a_{t+1} = a \mid \sigma_t(a) - \sigma_t(a') < - \gamma_t \cdot t) \leq \gamma_t
	\end{equation*}
	for $\gamma_t = \frac{\sqrt{8 \log (t)} \cdot k}{ \epsilon_t \cdot \sqrt{t} } + \frac{\sqrt{\log(t)}}{\eta_t \cdot \sqrt{t}} + \frac{\epsilon_t}{k} + 4/t + \exp(- \sqrt{t \log t}) = \frac{\sqrt{8 \log (t)} \cdot k}{ \epsilon_t \cdot \sqrt{t} } + \frac{\sqrt{\log(t)}}{\eta_t \cdot \sqrt{t}} + \frac{\epsilon_t}{k} + \mathcal{O}(1/t)$.
	
	We define the event $E \equiv \{ \hat{\sigma}_t(a) - \hat{\sigma}_t(a') < 2 M_t - \gamma_t \cdot t \}$, where $M_t > 0$ is a term that we detail later.
	With this event, we can express the probability that the algorithm plays $a$ at time $t + 1$ as
	\begin{equation*}
		\begin{aligned}
			&\mathbb{P}_{t+1}(a_{t+1} = a \mid \sigma_t(a) - \sigma_t(a') < - \gamma_t \cdot t) \\
			&\qquad = \mathbb{P}_{t + 1}(a_{t + 1} = a, E \mid \sigma_t(a) - \sigma_t(a') < - \gamma_t \cdot t) + \mathbb{P}_{t + 1}(a_{t + 1} = a, \neg E \mid \sigma_t(a) - \sigma_t(a') < - \gamma_t \cdot t) \\
			&\qquad \leq \underbrace{\mathbb{P}_{t + 1}(a_{t + 1} = a \mid E, \sigma_t(a) - \sigma_t(a') < - \gamma_t \cdot t)}_{(I)} + \underbrace{\mathbb{P}_{t + 1}(\neg E \mid \sigma_t(a) - \sigma_t(a') < - \gamma_t \cdot t)}_{(II)} \\
		\end{aligned}
	\end{equation*}
	Let us bound the individual terms in the following.
	
	We start with $(II)$. For any $M_t > 0$, we can bound the errors using the Azuma-Hoeffding inequality as in the proof of \cref{prop:mean-based-bandit-algorithm}:
	\begin{equation*}
		\mathbb{P}(\vert \hat{\sigma}_t(a) - \sigma_t(a) \vert \geq M_t) \leq 2 \exp \left( - \frac{M_t^2}{2 \sum_{\tau = 1}^{t} (k / {\epsilon_\tau})^2} \right) \leq 2 \exp \left( - \frac{1}{2t} (M_t \epsilon_t / k)^2 \right).
	\end{equation*}
	Therefore, with $M_t$ such that $M_t \epsilon_t = k \cdot \sqrt{2 t \log(t)}$, we get
	\begin{equation*}
		\begin{aligned}
			\mathbb{P}_{t + 1}(\neg E \mid \sigma_t(a) - \sigma_t(a') < - \gamma_t \cdot t)
			&=\mathbb{P} \left( \hat{\sigma}_t(a) - \hat{\sigma}_t(a') \geq 2 M_t - \gamma_t \cdot t \mid \sigma_t(a) - \sigma_t(a') < - \gamma_t \cdot t \right) \\
			&\leq \mathbb{P}((\hat{\sigma}_t(a) - \sigma_t(a)) - (\sigma_t(a') - \hat{\sigma}_t(a'))> 2 M_t) \\
			&\leq \mathbb{P}( \vert \hat{\sigma}_t(a) - \sigma_t(a) \vert + \vert \sigma_t(a') - \hat{\sigma}_t(a') \vert > 2 M_t) \\
			&\leq \mathbb{P}( \vert \hat{\sigma}_t(a) - \sigma_t(a) \vert > M_t \vee \vert \sigma_t(a') - \hat{\sigma}_t(a') \vert > M_t) \\
			&\leq 2 \cdot 2 \exp \left( - \frac{1}{2t} (M_t \epsilon_t / k)^2 \right) \\
			&\leq \frac{4}{t}.
		\end{aligned}
	\end{equation*}
	
	Now, we continue with $(I)$. On the condition that $\sigma_t(a) - \sigma_t(a') < - \gamma_t \cdot t$ (we drop this here for better readability), we find
	\begin{equation*}
		\begin{aligned}
			\mathbb{P}_{t+1} \left( a_{t+1} = a \mid E \right) 
			&= \mathbb{P}_{t+1} \left( a_{t+1} = a \mid \hat{\sigma}_t(a) - \hat{\sigma}_t(a') < 2 M_t - \gamma_t \cdot t \right) \\
			&= (1 - \epsilon_t) \frac{w_t(a)}{\sum_{\tilde{a}} w_t(\tilde{a})} + \frac{\epsilon_t}{k} \\
			&\leq \frac{w_t(a)}{w_t(a')} + \frac{\epsilon_t}{k} \\
			&= \exp(\eta_t \cdot (\hat{\sigma}_t(a) - \hat{\sigma}_t(a'))) + \frac{\epsilon_t}{k} \\
			&< \exp(\eta_t \cdot (2 M_t - \gamma_t \cdot t)) + \frac{\epsilon_t}{k} \\
			&\leq \exp \left(\eta_t \cdot \left( 2 k \sqrt{2 t \log(t)} / \epsilon_t - \gamma_t \cdot t \right) \right) + \frac{\epsilon_t}{k} \\
			&\leq \exp \left(\eta_t \cdot \left( k \cdot\sqrt{8 t \log(t)} / \epsilon_t - \left( \frac{\sqrt{8 \log (t)} \cdot k}{ \epsilon_t \cdot \sqrt{t} } + \frac{\sqrt{\log(t)}}{\eta_t \cdot \sqrt{t}} + \frac{\epsilon_t}{k} \right) \cdot t \right) \right) + \frac{\epsilon_t}{k} \\
			&\leq \exp \left(- \sqrt{t \cdot \log(t)} - \frac{\epsilon_t}{k} \cdot \eta_t \cdot t \right) + \epsilon_t \leq \exp \left( - \sqrt{t \cdot \log(t)} \right) + \frac{\epsilon_t}{k}
		\end{aligned}
	\end{equation*}
	
	In summary, we find
	\begin{equation*}
		\begin{aligned}
			\mathbb{P}_{t+1}(a_{t+1} = a \mid \sigma_t(a) - \sigma_t(a') < - \gamma_t \cdot t) 
			&\leq (I) + (II)  \\
			&\leq \exp \left( - \sqrt{t \cdot \log(t)} \right) + \frac{\epsilon_t}{k} + 4 / t \\
			&\leq \gamma_t
		\end{aligned}
	\end{equation*}
	This is less than or equal to $\gamma_t$, which concludes the proof.

	With $\eta_t = \epsilon_t = \min\{1, \sqrt[4]{\log(t) / t} \}$, we get
	\begin{equation*}
		\gamma_t = (2 \sqrt{2} \cdot k + 1 + 1 / k) \cdot \sqrt[4]{\frac{\log(t)}{t}} + \mathcal{O}(1/t)
	\end{equation*}
	for $t$ such that $\epsilon_t < 1$, which means that $\gamma_t \in \mathcal{O}(\sqrt[4]{\log(t) / t})$. Again, this is decreasing for $t \geq 2$, which we find by analyzing the derivative with respect to $t$.
\end{proof}

\subsection{Regret bound for Exp3}
\label{sec:proof-regret-bound-Exp3}

In the following, we prove the theorem below. 

\begin{theorem}
	\label{thm:exp3-regret-bound}
	Consider the Exp3 algorithm with bandit feedback as defined in \cref{sec:mean-based-Exp3}. For any $a \in [k]$, the algorithm has the following expected regret bound:
	\begin{equation*}
		R_T(a) \leq \sum_{t = 1}^T \left( \frac{\eta_t}{\epsilon_t} + \epsilon_t \right) + \frac{1}{\eta_{T + 1}}.
	\end{equation*}
\end{theorem}

Assume $\eta_t = t^{-\alpha}$ and $\epsilon_t = t^{-\beta}$. If $0 < \beta < \alpha < 1$, the regret grows sublinearly, and we have established that the algorithm is no-external-regret.

We prove the theorem by following standard follow-the-perturbed-leader (FTRL) regret bounds, as given by \citet[Exercises~28.12,~28.13]{lattimore_bandit_2020} and \citet[Section~2.3]{shalev-shwartz_online_2011}. Here, we combine an anytime approach (decaying learning rates), explicit exploration ($\epsilon_t > 0$), bandit feedback, and a maximization problem (instead of the usual minimization problem). 

\subsubsection{The FTRL algorithm and other preliminaries}

For now, let's assume the following full-feedback scenario: Let $y_t \in \R^k$ be a reward vector, and let $F_t: \Delta(k) \to \R$ be a sequence of convex, Legendre\footnote{See, e.g., \citet[Section~26.4]{lattimore_bandit_2020}} regularization functions $F_t = (1/\eta_t) \cdot F$. The \textit{FTRL algorithm }chooses action $p_{t + 1} \in \Delta(k)$ according to
\begin{equation*}
	p_{t + 1} = \argmax_{p \in \Delta(k)} \phi_{t + 1}(p) := \argmax_{p \in \Delta(k)} \sum_{s = 1}^{t} \langle y_s, p \rangle - F_{t + 1}(p).
\end{equation*}

We will consider the scaled negative entropy regularizer $F(p) = - h(p)$ where the entropy is given by
\begin{equation*}
	h(p) = - \sum_{i = 1}^k p_i \log(p_i).
\end{equation*}
It is well-known that this allows to express the FTRL iteration in closed form in terms of the MWU/Exp3 update, which relates this algorithm to ours:
\begin{equation*}
	p_{t + 1}(a) = \frac{\exp(\eta_t \sum_{s = 1}^t y_{s}(a))}{\sum_{a' = 1}^k \exp(\eta_t \sum_{s = 1}^t y_{s}(a'))}.
\end{equation*}

The \textit{Bregman divergence} of a function $F$ for two vectors $p, q$ is the difference between $F(p)$ and the first-order Taylor expansion of $F$ at $q$, evaluated at $p$. For any values $p$, $q$, the Bregman divergence is non-negative:
\begin{equation*}
	D_F(p, q) = F(p) - F(q) - \langle \nabla F(q), p - q \rangle.
\end{equation*}

\subsubsection{Derivation of the regret bound}

Many of the arguments below are standard in the online learning literature. Again, \citet{lattimore_bandit_2020} and \citet{shalev-shwartz_online_2011} provide recommendable entry points. Here, we just give the rough outline of the regret bound derivation.

We begin with the following lemma, which provides a regret bound for FTRL \textit{without exploration in the full-feedback setting}

\begin{lemma}
	\label{lem:standard-FTRL-regret-bound}
	The FTRL algorithm without explicit exploration has regret
	\begin{equation*}
		\begin{aligned}
			R_T(p) &\leq \sum_{t = 1}^{T} \left( \langle p_{t+1} - p_t, y_t \rangle - \frac{D_{F}(p_{t+1}, p_t)}{\eta_t} \right) + \frac{F(p) - \min_{p' \in \Delta(k)} F(p')}{\eta_{T + 1}} \\
			&\leq \sum_{t = 1}^T \frac{\eta_t}{2} \Vert y_t \Vert_{(\nabla^2 F(z))^{-1}}^2 + \frac{F(p) - \min_{p' \in \Delta(k)} F(p')}{\eta_{T + 1}}
		\end{aligned}
	\end{equation*}
	Here, $z \in [p_t, p_{t+1}]$ is a vector such that $D_{F}(p_{t+1}, p_t) = \frac{1}{2} \Vert p_{t+1} - p_t \Vert_{(\nabla^2 F(z))^{-1}}^2$.
\end{lemma}

The proof is analogous to the proof of \citet[Exercise~28.12]{lattimore_bandit_2020}\footnote{Solution available online: \url{https://tor-lattimore.com/downloads/book/solutions.pdf}} and applies Theorem 26.13 of \citet{lattimore_bandit_2020} to reach the final form.

If the algorithm \textit{samples} actions from $a_t \sim p_t$ and only has access to $\mathcal{F}_{t-1}$-measurable and conditionally unbiased reward estimates $\hat{y}_t$ ($\E[\hat{y}_t \mid \mathcal{F}_{t-1}] = y_t$), things only change slightly. Let $\xi_a \in \Delta(k)$ denote the distribution that places all weight on action $a$. Moving to expected regret, we get by the tower property that
\begin{equation*}
	\mathbb{E}[R_T(a)] = \mathbb{E} \left[ \sum_{t = 1}^{T} y_t(a) - y_t(a_t) \right] = \mathbb{E} \left[ \sum_{t = 1}^{T} \langle \xi_a - p_t, y_t \rangle \right] = \mathbb{E} \left[ \sum_{t = 1}^{T} \langle \xi_a - p_t, \hat{y}_t \rangle \right].
\end{equation*}
The expected regret bound hence remains the same when replacing $y_t$ by $\hat{y}_t$ in \cref{lem:standard-FTRL-regret-bound} above. What changes is $\Vert \hat{y}_t \Vert_{\nabla^2 F(z)^{-1}}^2$, which leads to a variance-dependent term. 

We now also add \textit{explicit exploration} by sampling from 
\begin{equation*}
	\hat{p}_{t+1} = (1 - \epsilon_t) \cdot p_{t+1} + \epsilon_t \mu,
\end{equation*}
where $\mu = (1/k, \dots, 1/k) \in \Delta(k)$ represents a uniform distribution. 
As in the main part of this paper, we continue using conditionally unbiased reward estimates
\begin{equation*}
	\hat{y}_t(a) = \frac{y_t(a)}{\hat{p}_t(a)} \cdot \mathbb{I}(a_t = a).
\end{equation*}

\begin{lemma}
	The FTRL algorithm with bandit feedback, conditionally unbiased estimates $\hat{y}_t$, and explicit exploration $\epsilon_t$ has expected regret bound
	\begin{equation*}
		\begin{aligned}
			\mathbb{E}[R_T(a)] &\leq \sum_{t = 1}^T \frac{\eta_t}{2} \Vert \hat{y}_t \Vert_{(\nabla^2 F(z))^{-1}}^2 + \frac{F(p) - \min_{p' \in \Delta(k)} F(p')}{\eta_{T + 1}} + \sum_{t=1}^T \epsilon_t \\
			&\leq \sum_{t=1}^T (\frac{\eta_t}{\epsilon_t} + \epsilon_t) + \frac{F(p) - \min_{p' \in \Delta(k)} F(p')}{\eta_{T + 1}} 
		\end{aligned}        
	\end{equation*}
	Here, $z \in [p_t, p_{t+1}]$ is a vector such that $D_{F}(p_{t+1}, p_t) = \frac{1}{2} \Vert p_{t+1} - p_t \Vert_{(\nabla^2 F(z))^{-1}}^2$.
\end{lemma}

In the first step of the lemma, the regret is combined from the FTRL algorithm and uniform exploration and follows from linearity. The second step replaces $\Vert \hat{y}_t \Vert_{(\nabla^2 F(z))^{-1}}^2$ according to
\begin{equation*}
	\begin{aligned}
		\mathbb{E} \left[ \frac{\eta_t}{2} \Vert \hat{y}_t \Vert_{(\nabla^2 F(z))^{-1}}^2 \right]
		&= \mathbb{E} \left[ \frac{\eta_t z(a_t) (\hat{y}_t(a_t))^2}{2} \right] \\
		&= \mathbb{E} \left[ \mathbb{E}_{t-1} \left[ \frac{\eta_t z(a_t) (\hat{y}_t(a_t))^2}{2} \right] \right] \\
		&= \mathbb{E} \left[ \sum_{a} p_t(a) \frac{\eta_t z(a) (y_t(a))^2}{2 (p_t(a))^2} \right] \\
		&\leq \mathbb{E} \left[ \sum_a \frac{\eta_t z(a)}{2 p_t(a)} \right] \\
		&\leq \mathbb{E} \left[ \frac{\eta_t}{2 \epsilon_t} \sum_a z(a) \right] \\
		&\leq \frac{\eta_t}{\epsilon_t} \\
	\end{aligned}
\end{equation*}
Here, we used $H = \nabla^2 F_t(z) = \nabla^2 \frac{1}{\eta_t}h(z) = \text{diag}(1/(\eta_t z))$ and $z \in [p_t, p_{t+1}]$.

\cref{thm:exp3-regret-bound} follows with
\begin{equation*}
	F(p) - \min_{p' \in \Delta(k)} F(p') = - h(p) - \min_{p' \in \Delta(k)} (- h(p')) = \max_{p' \in \Delta(k)} h(p') - h(p) \leq 1.
\end{equation*}

\section{Exploitation and regret}
\label{sec:mean-based-exploitation-and-regret}

In this section, we argue that the exploitability of mean-based algorithms does not necessarily indicate that they suffer regret. To align our arguments with the existing literature, we consider the original definition of mean-based algorithms by \citet{braverman_selling_2018, deng_strategizing_2019}, which is based on a fixed time horizon $T$:

\begin{definition}[Finite-horizon mean-based algorithms, \citet{braverman_selling_2018, deng_strategizing_2019}]
	An algorithm is $\gamma$-mean-based if whenever $\sigma_t(a') - \sigma_t(a) > \gamma T$, the probability that the algorithm chooses action $a$ in round $t$ is at most $\gamma$. The algorithm is mean-based if it is $\gamma$-mean-based for some $\gamma \in o(T)$.
\end{definition}

We consider the following game, introduced by \citet{deng_strategizing_2019}. They state that an optimizer strategizing against a mean-based algorithm would play action "Top" for the first $T/2$ rounds, then commit to action "Bottom" for the remaining rounds. The optimizer's reward exceeds the Stackelberg reward by a positive margin that grows with $T$.

\begin{center}
	\begin{tabular}{| c | c | c | c |}
		\hline
		& Left & Mid & Right \\
		\hline
		Top    & $(0, \sqrt{\gamma})$ & $(-2, -1)$ & $(-2, 0)$ \\
		\hline
		Bottom  & $(0, -1)$ & $(-2, 1)$ & $(2, 0)$ \\
		\hline
	\end{tabular}
\end{center}

Here, we show that the learner with mean-based algorithm is not bound to suffer linear regret in this setting. 
We need to consider the critical timesteps at which the mean-based condition "changes". This either means that an action becomes dominated by another action or that it ceases to be dominated. For this setting, these times are stated in the table below (c.f. \citep{deng_strategizing_2019}).

\begin{center}
	\begin{tabular}{ r l | c c c }
		\toprule
		\multicolumn{2}{c}{Time period} & Preferred action & Optimizer reward & Learner reward \\
		\midrule
		$0$ & to $\sqrt{\gamma} \cdot T$ & ? & ? ($\geq -2$) & ? ($\geq -1$) \\
		$\sqrt{\gamma} \cdot T$ & to $\frac{1}{2}T$ & Left & $0$ & $\sqrt{\gamma}$ \\
		$\frac{1}{2}T$ & to $(\frac{1 + \sqrt{\gamma}}{2} - \gamma) \cdot T$ & Left & $0$ & $-1$ \\
		$(\frac{1 + \sqrt{\gamma}}{2} - \gamma) \cdot T$ & to $(\frac{1 + \sqrt{\gamma}}{2} + \gamma) \cdot T$ & ? & ? ($\geq -2$) & ? ($\geq -1$) \\
		$(\frac{1 + \sqrt{\gamma}}{2} + \gamma) \cdot T$ & to $(1 - \gamma) \cdot T$ & Left & $2$ & $0$ \\
		$(1 - \gamma) \cdot T$ & to $T$ & ? & ? ($\geq -2)$ & ? ($\geq -1$) \\
		\bottomrule
	\end{tabular}
\end{center}

The "Preferred action" is an action that dominates all other actions in the mean-based sense, meaning that this action will be played with probability of at least $1 - 2 \gamma$. Whenever we placed a question mark in the table, there is no such action, and we cannot make a clear statement about the played action. Instead, we provide a conservative lower bound on the rewards.

The best action in hindsight is Mid or Right, giving a reward of zero in total. The mean-based algorithm receives reward
\begin{equation*}
	\begin{aligned}
		\sum_{t = 1}^T x_t(a_t) 
		&\geq \sqrt{\gamma} T \cdot (-1) 
		+ (\frac{1}{2} - \sqrt{\gamma}) T \cdot 0 \\
		& \quad + (\frac{1 + \sqrt{\gamma}}{2} - \gamma - \frac{1}{2}) T \cdot (-1)
		+ (\frac{1 + \sqrt{\gamma}}{2} + \gamma - \frac{1 + \sqrt{\gamma}}{2} + \gamma) T \cdot (-1) \\
		& \quad + (1 - \gamma - \frac{1 + \sqrt{\gamma}}{2} - \gamma) T \cdot 0 
		+ (1 - 1 + \gamma) T \cdot (-1) \\
		&= \sqrt{\gamma} T \cdot (-1) 
		+ (\frac{\sqrt{\gamma}}{2} - \gamma) T \cdot (-1)
		+ 2 \gamma T \cdot (-1) \\
		& \quad + \gamma T \cdot (-1) \\
		&= - T \cdot ( \frac{3}{2} \sqrt{\gamma} + 2 \gamma) \\
		&\in - o(T)
	\end{aligned}
\end{equation*}

This means that the algorithm experiences sub-linear regret in this game. 

We conclude that these algorithms are not necessarily "exploitable" but instead are more "predictable" than default no-regret learners. Similar to Stackelberg games, this first-mover commitment might not be to their disadvantage, and some papers even say that the mean-based algorithm benefits along with the optimizer \citet{hartline_regulation_2025, arunachaleswaran_algorithmic_2024, lin_generalized_2025}. 

\end{document}